\documentclass[preprint,12pt]{elsarticle}

\usepackage{times}
\usepackage{soul}
\usepackage{xurl}
\usepackage[switch]{lineno}

\usepackage{graphicx} 
\usepackage{subcaption}
\usepackage{nomencl}
\usepackage{etoolbox}
\usepackage{amssymb}
\usepackage{hyperref}
\usepackage[utf8]{inputenc}
\usepackage[T1]{fontenc}
\usepackage{float}
\usepackage{color}
\usepackage{lineno}
\usepackage{amsmath}
\usepackage[linesnumbered]{algorithm2e}
\usepackage{caption}
\usepackage{booktabs}
\usepackage{array}
\usepackage{pgfplots}
\usepackage[euler-digits]{eulervm}
\usepackage[eulergreek]{sansmath}
\usepackage{amsmath} 
\usepackage{amsthm}

\usepackage{algorithmic}
\usepackage{amsfonts}
\usepackage{subcaption}

\usepackage{float}
\usepackage[hypcap]{caption}
\usepackage{subcaption}
\usepackage{algorithm2e}

\usetikzlibrary{positioning}
\pgfplotsset{compat=1.18}
\newlength\algowd

\newcommand{\Prod}[3][L^2]{\left(#2, #3\right)_{#1}}

\newcommand{\deriv}[2]{\frac{\partial {#1}}{\partial {#2}}}

\newtheorem{lemma}{Lemma}
\newtheorem{theorem}{Theorem}

\DeclareMathOperator*{\Span}{span}

\newcommand{\Vsigma}{\boldsymbol{\sigma}}
\newcommand{\Vtau}{\boldsymbol{\tau}}

\newcommand{\Norm}[2][]{\left\|#2\right\|_{#1}}

\begin{document}

\begin{frontmatter}

\title{Collocation-based Robust Variational Physics-Informed Neural Networks (CRVPINN)}

\author{Marcin \L{}o\'s$^1$,
Tomasz S\l{}u\.zalec$^1$,
Pawe\l{} Maczuga$^1$,
Askold Vilkha$^1$, \\
Carlos Uriarte$^2$, 
Maciej Paszy\'nski$^1$}

\address{$^{(1)}$AGH University of Krakow,
Poland \\
e-mail:  \{los,pmaczuga,maciej.paszynski\}@agh.edu.pl \\
$^{(2)}$ University of the Basque Country (UPV/EHU), Leioa, Spain \\
e-mail: carlos.uriarte@ehu.eus}

\begin{abstract}
Physics-Informed Neural Networks (PINNs) have been successfully applied to solve Partial Differential Equations (PDEs). Their loss function is founded on a strong residual minimization scheme. Variational Physics-Informed Neural Networks (VPINNs) are their natural extension to weak variational settings. In this context, the recent work of Robust Variational Physics-Informed Neural Networks (RVPINNs) highlights the importance of conveniently translating the norms of the underlying continuum-level spaces to the discrete level. Otherwise, VPINNs might become unrobust, implying that residual minimization might be highly uncorrelated with a desired minimization of the error in the energy norm. However, applying this robustness to VPINNs typically entails dealing with the inverse of a Gram matrix, usually producing slow convergence speeds during training. 
In this work, we accelerate the implementation of RVPINN, establishing a LU factorization of sparse Gram matrix in a kind of point-collocation scheme with the same spirit as original PINNs. 
We call out method the Collocation-based Robust Variational Physics Informed Neural Networks (CRVPINN). We test our efficient CRVPINN algorithm on Laplace, advection-diffusion, and Stokes problems in two spatial dimensions.
\end{abstract}
	

\end{frontmatter}

\section{Introduction}
The extraordinary success of Deep Learning (DL) algorithms in various scientific fields \cite{c17,c28,c13} over the last decade has recently led to the exploration of the possible applications of (deep) neural networks (NN) for solving partial differential equations (PDEs). The exponential growth of interest in these techniques started with the Physics-Informed Neural Networks (PINN) (\cite{c37}). This method takes into account the physical laws described by PDEs during the learning process. The network is trained using the strong residual evaluated at the set of points selected in the computational domain and its boundary. PINNs have been successfully applied to solve a wide range of problems, from fluid mechanics \cite{c3,c34}, in particular Navier-Stokes equations \cite{c29,c42,c45}, wave propagation \cite{c38,c33,c12}, phase-filed modeling \cite{c15}, biomechanics \cite{c1,c27}, quantum mechanics \cite{c21}, electrical engineering \cite{c36}, problems with point singularities \cite{c19}, uncertainty qualification \cite{c46}, dynamic systems \cite{c44,c24}, or inverse problems \cite{c6,c35,c32}, among many others.

A natural continuation to PINNs into the concept of weak residuals is the so-called Variational PINNs (VPINNs) \cite{c22}. VPINN employs a variational loss function to minimize during the training process. 
The VPINN method has also found several applications, from Poisson and advection-diffusion equations \cite{c23}, non-equilibrium evolution equations \cite{c18}, solid mechanics \cite{c30}, fluid flow \cite{c25}, and inverse problems \cite{c31,c2}, among others.
Recently, the the Robust Variational Physics Informed Neural Networks (RVPINNs), has been proposed in \cite{c40}. The authors consider the modified loss multiplies with the inverse of the Gram matrix. 
The RVPINN introduces the robust loss function, that is the lower and upper bound for the true error. In other words, while solving PDEs we
monitor the training of the RVPINNs using the robust loss, and we know the quality of the trained solution. Thus, we know when to stop the training. This is especially true when we do not know the exact solution, we can control the quality of the solution, by looking at the value of the robust loss computer during the training.

However, RVPINN requires expensive integration for weak residuals, and the Gram matrix is also computing with expensive integrals. 
In this paper, we develop efficient collocation method for RVPINN.
{In our approach we replace the continuous domain $\Omega$ by the discrete set of collocation points $\Omega_h$. We derive discrete weak formulations, using the Kronecker deltas as test functions. 
We redefine the discrete integration by parts, and discrete Poincare inequality.
In this discrete setup, we show that discrete weak formulation are continuous, and inf-sup stable.
Using the discrete inner product related with the inf-sup stability norm, and the Kronecker delta test functions, we derive the Gram matrix, that has a similar structure as for the finite difference method \cite{c41}.
We introduce a robust loss function in terms of the discrete residuals and the inverse of the discrete Gram matrix.
We compute the Gram matrix resulting from the discrete inner product, and we define our robust loss function as the inverse of the Gram matrix multiplied by the square of the discrete residual. 
The benefit of our method is that the same Gram matrix
and its inverse can be used with a large class of PDEs, especially since it does not depend
on the right-hand side.
For example, we can use an identical inverse of the Gram matrix with diffusion problems having different right-hand sides and advection-dominated diffusion problems with different right-hand sides.}

In our approach, we show how to use the knowledge developed for RVPINN without expensive integrations, using just points during the training process.

{We also show that our collocation method for RVPINN loss does not require computing of the inverse of the Gram matrix, since  we only compute the action of the inverse, which can be replaced by a solution of system of linear equations. The Gram matrix ${\bf G}$ can be LU factorized once at the beginning, and in each iteration we can perform forward and backward substitutions in a linear computational cost.}

The article is organized as follows. We start in Section 2 with description of the theoretical background for definition of the collocation method for RVPINN. 
Section 3 introduces four computational examples, namely the Laplace problem with the exact solution being the tensor product of sin functions, the Laplace problem with the exact solution being a combination of exponent and sin functions, the Poisson problem with non-constant diffusion, the Poisson problem with a jump, the advection-diffusion problem, and the Stokes problem. This section aims to illustrate how to define the robust loss functions for the exemplary problems.
Section 4 discusses the computational costs of collocation method for RVPINN. 
Finally, Section 6 presents the colab implementation of the method. The paper is concluded in Section 7.

\section{Abstract framework}\label{sec:abstract}

For concreteness, we consider the case of the 2D unit square domain~$(0, 1)^2$ with spatial resolution~$N \in \mathbb N$ and corresponding set of uniformly distributed collocation points
\begin{equation}
\Omega_h := \{ [ih,jh] \in (0,1)^2: 0 \leq i \leq N, 0 \leq j \leq N\},
\end{equation} where~$h=1/N$ denotes the discretization size. We consider
\begin{equation}
D_h := \{ u:\Omega_h\longrightarrow\mathbb{R}\}\cong \mathbb{R}^{(N+1)^2}
\end{equation} and equip it with the following discrete inner product and induced norm:
\begin{equation}
(u,v)_h := h^2 \sum_{p\in \Omega_h} u(p)v(p),\qquad \|u\|_h^2 :=(u,u)_h, \qquad u,v\in D_h.
\end{equation}

Now, we follow the notation convention below for simplicity:
\begin{equation}
u_{i,j}:=u(ih,jh),\qquad 0\leq i,j\leq N.
\end{equation} We introduce a canonical orthonormal basis for $D_h$ given by a set of functions $\delta_{i,j}:\Omega\longrightarrow\mathbb{R}$, $0\leq i,j\leq N$, that behave as Kronecker deltas over $\Omega_h$,
\begin{equation}
\label{eq:test}
\delta_{i,j}(x)=
\displaystyle{\left\{\begin{aligned}
&1 \quad \textrm{ if }x=x_{i,j}, \\
&0 \quad \textrm{ if }x\neq x_{i,j}.
\end{aligned}
\right.}
\end{equation} 

Following the spirit of finite differences, we consider the finite gradient operations given by
\begin{align}
\nabla_{+} u_{i,j} := \left(\nabla_{x+} u_{i,j},\nabla_{y+} u_{i,j}\right) := 
\left(\frac{u_{i+1,j}-u_{i,j}}{h},\frac{u_{i,j+1}-u_{i,j}}{h}\right),\\
\nabla_{-} u_{i,j} := \left(\nabla_{x-} u_{i,j},\nabla_{y-} u_{i,j}\right) := 
\left(\frac{u_{i,j}-u_{i-1,j}}{h},\frac{u_{i,j}-u_{i,j-1}}{h}\right),
\end{align} for $0\leq i\pm 1, j\pm 1\leq N$. As a result, we can define the following discrete inner product according to these gradient values: 
\begin{align}
\label{eq:H1}
(u,v)_{\nabla,h} &:= 
(\nabla_{x+}u,\nabla_{x+}v)_h + 
(\nabla_{y+}u,\nabla_{y+}v)_h \\
&\phantom{:}= (\nabla_{x-}u,\nabla_{x-}v)_h + 
(\nabla_{y-}u,\nabla_{y-}v)_h,
\end{align} with corresponding induced norm
\begin{equation}
\label{eq:H1norm}
\|u\|_{\nabla,h}^2 :=(u,u)_{\nabla,h} = \Vert \nabla_{x+} u\Vert_h^2 + \Vert \nabla_{y+} u\Vert_h^2.
\end{equation} 

At this point, it is convenient to establish the discretization space that possesses homogeneous boundary conditions:
\begin{equation}
D_{0,h} = \{ u\in D_h: u\vert_{\partial\Omega}=0\}.
\end{equation} Thus, $D_{0,h}$ is an $(N-1)^2$-dimensional space with basis $\{\delta_{i,j}\}_{0<i,j<N}$.

Below, we show three properties of interest in this discrete gradient setting.

\begin{lemma}[Discrete integration by parts]
\label{lem:integration-by-parts}
Given~$u,v\in D_{0,h}$, it satisfies
\begin{equation}
\begin{aligned}
  \left(\nabla_{x+} u, v\right)_h &=
  - \left(u, \nabla_{x-} v\right)_h \\
  \left(\nabla_{y+} u, v\right)_h &=
  - \left(u, \nabla_{y-} v\right)_h
\end{aligned}
\end{equation}
\end{lemma}
\begin{proof}
By bilinearity of the inner product, it is enough to show it for $u = \delta_{i,j}$.
We have
\begin{equation}
  \nabla_{x+} \delta_{i,j} = 
  \begin{cases}
    \phantom{-}h^{-1} & \text{at } x_{i-1,j}, \\
    -h^{-1} & \text{at } x_{i, j}, \\
    0 & \text{elsewhere}.
  \end{cases}
\end{equation} Then,
%
%
\begin{equation}
\begin{aligned}
  \left(\nabla_{x+} \delta_{i,j}, v\right)_h &=
  h^{-1} v_{i-1, j} + \left(-h^{-1}\right) v_{i,j} \\&=
  - \nabla_{x-}v_{i,j} =
  -\left(\delta_{i,j}, \nabla_{x-}v\right)_h.
\end{aligned}
\end{equation} The proof is similar for the $y$-axis.
\end{proof}

\begin{lemma}[Discrete product rule]
Given~$u,v \in D_{0,h}$, it satisfies
\begin{equation}
\begin{aligned}    
  \nabla_{x+}(u v)_{i,j} &=
  u_{i+1,j}(\nabla_{x+}v)_{i,j} + (\nabla_{x+}u)_{i,j}v_{i,j}.\\
    \nabla_{y+}(u v)_{i,j} &=
  u_{i,j+1}(\nabla_{y+}v)_{i,j} + (\nabla_{y+}u)_{i,j}v_{i,j}.
\end{aligned}
\end{equation}
\end{lemma}
\begin{proof}
\begin{equation}
\begin{aligned}    
  \nabla_{x+}(u v)_{i,j} &= u_{i+1,j} v_{i+1,j} - u_{i,j}v_{i,j} \\&=
  u_{i+1,j}v_{i+1,j} - u_{i+1,j}v_{i,j} + u_{i+1,j}v_{i,j} - u_{i,j}v_{i,j} \\&=
  u_{i+1,j}(\nabla_{x+}v)_{i,j} + (\nabla_{x+}u)_{i,j}v_{i,j}.
\end{aligned}
\end{equation} The proof is similar for the $y$-axis.
\end{proof}

\begin{lemma}[Discrete norm equivalence]
\label{lem:equivalence}
There exist constants~$0<c< C$ such that 
\begin{equation}
  c \|u\|_{\nabla,h}\leq \|u\|_h \leq C \|u\|_{\nabla,h}, \qquad \forall u \in D_{0,h}.
\end{equation} The upper-bound inequality is typically referred to as Poincaré's inequality.
\end{lemma}
\begin{proof} Let $u\in D_{0,h}$. From the inequality $(a-b)^2\leq 2(a^2+b^2)$ for all $a,b\in\mathbb{R}$, we deduce 
\begin{align}
    \Vert u\Vert_{\nabla,h}^2 &= \sum_{i,j} (u_{i+1,j}-u_{i,j})^2 + (u_{i,j+1}-u_{i,j})^2 \\
    &\leq 2\sum_{i,j} u_{i+1,j}^2 + u_{i,j}^2 + u_{i,j+1}^2 + u_{i,j}^2 = 8\sum_{i,j} u_{i,j}^2,
\end{align} where the summation indexes run along $0< i,j < N$. This provides a lower-bound constant $c=\frac{h}{2\sqrt{2}}$ for our norm equivalence inequality. 

For the upper-bound constant, we divide the proof into two parts. First, consider the translation operator~$\tau_x \colon D_{0,h} \rightarrow D_{0,h}$ defined pointwise by
\begin{equation}
  (\tau_x u)_{i, j} = 
  \begin{cases}
    u_{i+1, j}, & \text{if }i < N, \\
    0, & \text{if } i=N.
  \end{cases}
\end{equation}
Applying Lemma 2, we can write~$\nabla_{x+}(uv) = \tau_x u(\nabla_{x+}v) + (\nabla_{x+}u) v$.
Moreover, $\|\tau_x u\|_h = \|u\|_h$ for all ~$u \in D_{0,h}$
since only homogeneous boundary values are shifted out of the domain.

Second, consider~$\phi \in D_{0,h}$ such that~$\phi_{i,j} = ih\leq 1$ for all $0<i,j<N$. Then, $\nabla_{x+}\phi=1$. Consequently, we can write
\begin{equation}
\begin{aligned}
  \|u\|_h^2 &=
  \left(1, u^2\right)_h = 
  \left(\nabla_{x+}\phi, u^2\right)_h =
  - \left(\phi, \nabla_{x+} u^2\right)_h = \\
  &= - \left(\phi, (u + \tau_x u) \nabla_{x+}u\right)_h.
\end{aligned}
\end{equation}
Applying Cauchy Schwartz's inequality, we obtain
\begin{equation}
\begin{aligned}
  \|u\|_h^2 &
  \leq
  \|\phi\|_h \; \Big\lbrace\|u\|_h + \|\tau_x u\|_h \Big\rbrace \; \|\nabla_{x+}u\|_h
  \\&=
  2\; \|\phi\|_h \; \|u\|_h \; \|\nabla_{x+}u\|_h,
\end{aligned}
\end{equation} which implies
\begin{equation}
  \|u\|_h \leq 2\|\phi\|_h \|\nabla_{x+}u\|_h \leq 2 \|\nabla_{x+}u\|_h \leq 2 \|u\|_{\nabla,h}.
\end{equation}
\end{proof}

Now, we focus on the following convection-diffusion model problem: find $u\in D_{0,h}$ such that
\begin{equation}
    \beta_x  \nabla_{x+}u + \beta_y \nabla_{y+} u - \epsilon\Delta_h u = f,
\end{equation} where $f,\beta_x,\beta_y \in D_{0,h}$ are given source and coefficient functions, and $\Delta_h$ is the discrete Laplacian defined pointwise as
\begin{equation}
    \Delta_h u_{i,j}:=\frac{u_{i+1,j}-2u_{i,j}+u_{i-1,j}}{h^2}+\frac{u_{i,j+1}-2u_{i,j}+u_{i,j-1}}{h^2}.
\end{equation} It is immediate to check that $\Delta_h = \nabla_+\nabla_- = \nabla_-\nabla_+$.

Testing with $v \in D_{0,h}$,
\begin{equation}
\begin{aligned}
(\beta_x \nabla_{x+}u + \beta_y \nabla_{y+}u
-\epsilon \Delta_h u,v)_{h} =
(f,v)_{h},
\end{aligned}
\label{eq:weak}
\end{equation} and applying Lemma~\ref{lem:integration-by-parts} to the term with
the Laplacian, we obtain the following discrete weak variational reformulation: find $u\in D_{0,h}$ such that
\begin{equation}
\begin{aligned}\label{eq:discweak}
\overbrace{(\beta_x \nabla_{x+}u + \beta_y \nabla_{y+}u, v)_{h} +
\epsilon(\nabla_{+}u, \nabla_{+}v)_{h}}^{b(u,v)} 
= \overbrace{(f,v)_{h}}^{l(v)},\qquad \forall v\in D_{0,h},
\end{aligned}
\end{equation} where $b$ and $l$ are the corresponding discrete bilinear and linear forms over $(D_{0,h})^2$ and $D_{0,h}$, respectively.

Below we show the conditions under which \eqref{eq:discweak} satisfies the hypotheses of the Lax-Milgram theorem.

\begin{lemma}[Boundedness of $b$]
\label{lem:discrete-boundedness}
For bounded coefficients $\beta_x$ and $\beta_y$, there exists $\mu>0$ such that
\begin{eqnarray}
b(u,v)\leq \mu \; \|u\|_{\nabla,h} \; \|v\|_{\nabla,h},\qquad \forall u, v \in D_{0,h}.
\end{eqnarray}
\end{lemma}

\begin{proof}
Following Cauchy Schwarz's inequality,
\begin{equation}
\begin{aligned}
  b(u, v) &\leq 2\|\beta\|_\infty \; \|u\|_{\nabla,h}\; \|v\|_{h} + \epsilon\|u\|_{\nabla,h} \; \|v\|_{\nabla,h},\qquad u,v\in D_{0,h},
\end{aligned}
\end{equation}
where~$\beta :=(\beta_x,\beta_y)$ and $\|\beta\|_\infty := \max\{|(\beta_x)_{i,j}|, |(\beta_y)_{i,j}|\}$. Applying Lemma~\ref{lem:equivalence}, there exists~$C > 0$ such that
\begin{equation}
  b(u, v) \leq \overbrace{(2C \|\beta\|_\infty + \epsilon)}^{\mu}\; \|u\|_{\nabla,h}\; \|v\|_{\nabla,h}.
\end{equation}
\end{proof}

For simplicity, from now on, we restrict to the case of constant convection coefficients $\vert\beta_x\vert,\vert\beta_y\vert\leq\infty$. Hence, $\Vert \beta\Vert_{\infty} = \max\{\vert\beta_x\vert, \vert\beta_y\vert\}$.

\begin{lemma}[Coercivity of $b$]
\label{lem:discrete-inf-sup}
There exists $\alpha > 0$ such that
\begin{eqnarray}
b(u,u) \geq \alpha \; \|u\|_{\nabla,h}^2,\qquad\forall u \in D_{0,h},
\end{eqnarray} whenever $\epsilon > 2C\Vert\beta\Vert_\infty$, where~$C > 0$ is the constant from Poincarè's inequality.
\end{lemma}

\begin{proof}
Applying Cauchy-Schwarz and Poincaré's inequality to the convection term, we obtain
\begin{equation}
    \vert (\beta_x\nabla_{x+} u, u)_h + (\beta_y \nabla_{y+} u, u)_h\vert \leq 2C \Vert\beta\Vert_\infty \Vert u\Vert_{\nabla,h}^2,
\end{equation} which implies
\begin{equation}
\begin{aligned}
  b(u, u) \geq \overbrace{\left(- 2C \Vert \beta\Vert_\infty +\epsilon\right)}^{\alpha}\|u\|^2_{\nabla,h}.
\end{aligned}
\end{equation}
\end{proof}

As a result, variational problem \eqref{eq:discweak} is well-posed and thus admits a unique solution in $D_{0,h}$.

Now, invoking the Riesz representation theorem, we have that for each $u \in D_{0,h}$, there exists a unique $r(u) \in D_{0,h}$ such that
\begin{equation}
  (r(u), v)_{\nabla,h} = b(u, v)-l(v), \qquad \forall v \in D_{0,h},
\end{equation} which relates with the norm of the error~$u-u_{\text{EXACT}}$,
where~$u_{\text{EXACT}}$ denotes the solution to the weak problem~\eqref{eq:discweak}, as follows:

\begin{theorem}[Robustness]
Let~$u \in D_{0,h}$ and let~$r(u)\in D_{0,h}$ be its residual representative. Then,
\begin{equation}
    \frac{1}{\mu}\|r(u)\|_{\nabla,h}
    \leq
    \|u - u_{\text{EXACT}}\|_{\nabla,h}
    \leq
    \frac{1}{\alpha}\|r(u)\|_{\nabla,h},
\end{equation}
where $\mu$ and $\alpha$ are the boundedness and coercivity constants of~$b$, respectively.
\end{theorem}
\begin{proof}
It follows immediately from the boundedness and coercivity constants of Lemmas \ref{lem:discrete-boundedness} and \ref{lem:discrete-inf-sup}.
\end{proof}

Built on the upper and lower error control provided by this theorem (robustness), we construct the (robust) loss function as follows:
\begin{equation}
  \text{LOSS}(u) = \|r(u)\|_{\nabla,h}^2 = b(u,r(u))-l(r(u)), \qquad u\in D_{0,h}.
\end{equation} Identifying $r(u)$ with its vector of coefficients $\mathbf{r}(u)\in \mathbb{R}^{(N-1)^2}$, we have that the vector $\text{RES}(u) = \{b(u, r(u)_{i,j})-l(r(u)_{i,j})\}_{0 < i,j< N}$ satisfies the following:
\begin{equation}
    \text{RES}(u) = \mathbf{G}\; \mathbf{r}(u),
\end{equation} where~$\mathbf{G}$ is the Gram matrix of the inner product $(\cdot,\cdot)_{\nabla,h}$. So,
\begin{equation}
\begin{aligned}
  \|r(u)\|_{\nabla,h}^2 &= \mathbf{r}(u)^T\; \mathbf{G}\; \mathbf{r}(u) = \text{RES}(u)^T\; \mathbf{G}^{-1} \; \text{RES}(u).
\end{aligned}
\end{equation}

We will now construct the Gram matrix employing the Kronecker delta test functions as follows:
\begin{eqnarray}
\mathbf{G}_{i,j;k,l}=
&h^{-2}\displaystyle{\left\{\begin{aligned}
&\phantom{-}\,\,4\ \quad \textrm{ for }(i,j)=(k,l) \\
&-1 \quad \textrm{ for }(k,l)\in \{(i+1,j),(i-1,j)\}\\
&-1 \quad \textrm{ for }(k,l)\in\{(i,j+1),(i,j-1)\} 
\end{aligned}
\right.} 
\label{eq:gram}
\end{eqnarray} As a consequence, the Gram matrix of the inner product of $D_{0,h}$ is sparse, and it can be efficiently inverted.

For a neural network $u_{\theta}$ parameterized via the set of trainable parameters $\theta$, by abuse of notation, we will replace $u_\theta$ with $\theta$ in the the argument of LOSS and RES

%
%

\section{Numerical results for the Collocation method for Robust Variational Physics Informed Neural Networks}
\label{sec:PINN2D}
In this section we solve four two-dimensional model problems by using collocation method for RVPINN \cite{c40} method. 

The neural network represents the solution
\begin{equation}
u_{\theta}(x_1,x_2)=NN(x_1,x_2)=A_n \sigma\left(A_{n-1}\sigma(...\sigma(A_1\begin{bmatrix}x_1 \\ x_2 \end{bmatrix}+B_1)...+B_{n-1}\right)+B_n
\end{equation}
where $A_i$ are matrices with weights, $B_i$ are vectors of biases, and $\sigma$ is the activation faction (e.g., the tanh activation function, among alternative possibilities \cite{c20,c33}).

\subsection{Two-dimensional Laplace problem with sin-sin right-hand side}
\label{sec:PINN2D1}

Given $\Omega=(0,1)^2\subset\mathbb R^2$ we seek the solution of the model problem with manufactured solution
\begin{equation}
 -\Delta u = f_1,
\end{equation}
with zero Dirichlet b.c. In this problem we select the solution
\begin{equation}
\label{eq:exact1}
u(x_1,x_2)=sin(2\pi x_1) sin(2 \pi x_2).
\end{equation}
In order to obtain this solution, we employ the manufactured solution technique. Namely, we compute
\begin{eqnarray}
f_1(x_1,x_2)=-\Delta u(x_1,x_2) 
=8 \pi^2 \sin(2 \pi x_1) \sin(2 \pi x_2).
\end{eqnarray}
We define the following residual function
\begin{equation}
\label{eq:RES1}
RES_1(\theta)=\Delta u({\bf x})+f_1({\bf x})
\end{equation}
We enforce the zero Dirichlet b.c. on the NN in a strong way,  following the ideas presented in \cite{c43}.

This time we define the following loss function for CRVPINN
\begin{equation}
LOSS(\theta)=RES_1^T(\theta)\times {\bf G}^{-1}\times RES_1(\theta)
\end{equation}
with $RES_1(\theta)$ defined by (\ref{eq:RES1}) and Gram matrix defined by (\ref{eq:gram}).

{The sparsity pattern of the Gram matrix ${\bf G}$ is presented in Figure \ref{fig:G}. 
For the computing of the CRVPINN loss function, we actually do not need to compute inverse of the matrix ${\bf G}$. The matrix ${\bf G}$ is sparse. We need to solve a system of equations and multiply two vectors
\begin{eqnarray}
\begin{aligned}
 {\bf G} z &= RES_1(\theta) \\
 LOSS(\theta) &= RES_1^T(\theta) z
\end{aligned}
\end{eqnarray}
where we can perform once the LU factorization ${\bf G}={\bf L}{\bf U}$ and use it for a class of computational problems. Then, in every iteration we perform forward and backward substitution which have a linear computational costs ${\cal O}(N)$, namely
\begin{eqnarray}
\begin{aligned}
 {\bf U} z &= RES_1(\theta), \\
 {\bf L} q &= z, \\
 LOSS(\theta) &= RES_1^T(\theta) q.
\end{aligned}
\end{eqnarray}
}

\begin{figure}
 \centering
\includegraphics[width=0.4\textwidth]{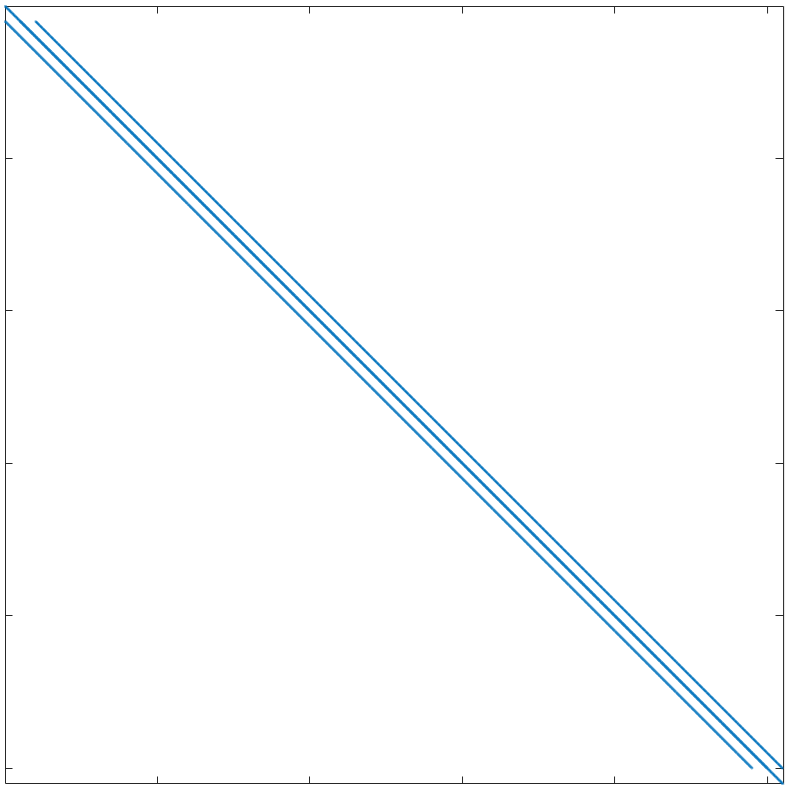}
\caption{Sparsity pattern of the Gram matrix ${\bf G}$ build with $H^1_0$ norm.}
\label{fig:G}
\end{figure}

\begin{figure}[!htb]
  \centering
\includegraphics[width=0.8\textwidth]{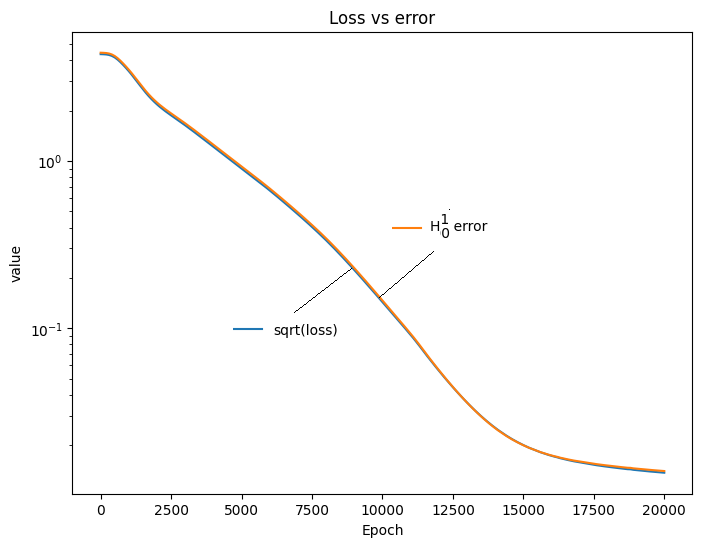}
  \caption{Convergence of CRVPINN and the true error $H^1_0(\Omega_h)$ for the Laplace problem with sin-sin right-hand side.}
  \label{fig:RPINN_Example1_convergence}
\end{figure}

\begin{figure}[!htb]
  \centering
\includegraphics[width=0.8\textwidth]{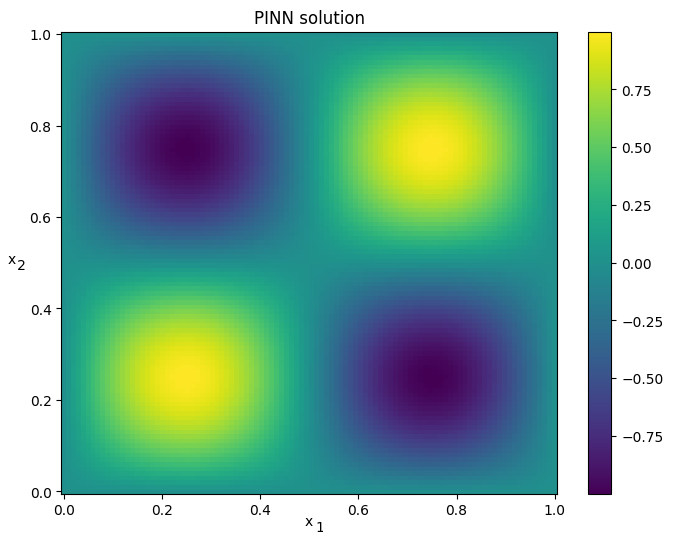}
  \caption{Solution obtained from CRVPINN for the Laplace problem with sin-sin right-hand side.}
  \label{fig:RPINN_Example1_result}
\end{figure}


The convergence of training with ADAM optimizer \cite{c26} is presented in Figure \ref{fig:RPINN_Example1_convergence}. We can see that our loss is robust and equal to the true error computed in (\ref{eq:H1}) norm. This is because for the Laplace problem we have $\mu=\alpha=1$ so $\sqrt{LOSS(\theta)}=\|u_{EXACT}-u_{\theta}\|_{H^1_0(\Omega_h)}$.
The obtained solution is presented in Figure \ref{fig:RPINN_Example1_result}.

\subsection{Two-dimensional Laplace problem with exp-sin right-hand side}
\label{sec:PINN2D2}

Given $\Omega=(0,1)^2\subset\mathbb R^2$ we seek the solution of the model problem with a manufactured solution
\begin{equation}
 -\Delta u = f_2,
\end{equation}
with zero Dirichlet b.c. In this problem, we select the solution
\begin{equation}
\label{eq:exact2}
u(x_1,x_2)=-e^{\pi (x_1-2x_2)}\sin(2\pi x_1)\sin(\pi x_2).
\end{equation}
In order to obtain this solution, we compute
\begin{eqnarray}
f_2(x_1,x_2)=-\Delta u(x_1,x_2) = \nonumber \\ 
= \pi^2 e^{\pi (x - 2 y)} \sin(\pi y) (4 \cos(2 \pi x) - 3 \sin(2 \pi x)) \\
- \pi^2 e^{\pi (x - 2 y)} \sin(2 \pi x) (4 \cos(\pi y) - 3 \sin(\pi y)) \nonumber
\end{eqnarray}

We define the following residual function
\begin{equation}
\label{eq:RES2}
RES_2(\theta)=\Delta u({\bf x})+f_2({\bf x})
\end{equation}
We enforce the zero Dirichlet b.c. on the NN in a strong way,  following the ideas presented in \cite{c43}.

We define the following loss function for CRVPINN
\begin{equation}
LOSS(\theta)=RES_2^T(\theta)\times {\bf G}^{-1}\times RES_2(\theta)
\end{equation}
with $RES_2(\theta)$ defined by (\ref{eq:RES2}) and Gram matrix defined by (\ref{eq:gram}).
{As in the previous example, the inverse of ${\bf G}$ can be replace by LU factorization, and linear cost forward and backward substitutions to obtain a linear computational cost overhead. }
\begin{figure}[!htb]
  \centering
\includegraphics[width=0.8\textwidth]{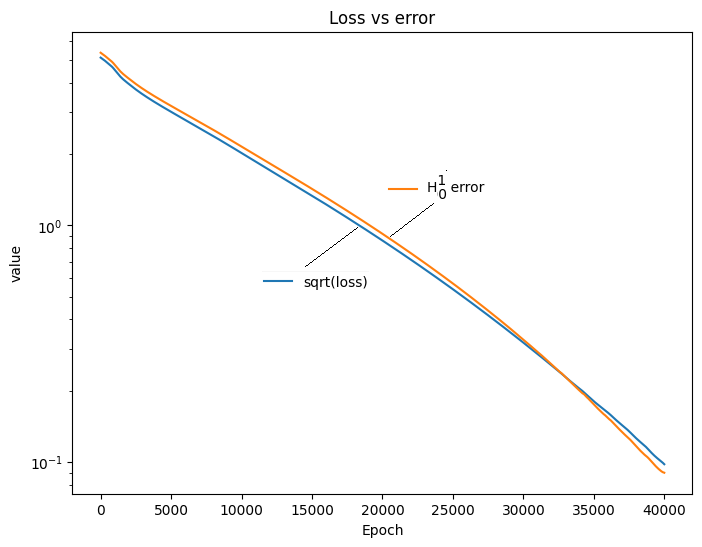}
  \caption{Convergence of CRVPINN and the true error $H^1_0(\Omega_h)$ for the Laplace problem with sin-exp right-hand side.}
  \label{fig:RPINN_Example2_convergence}
\end{figure}

\begin{figure}[!htb]
  \centering
\includegraphics[width=0.8\textwidth]{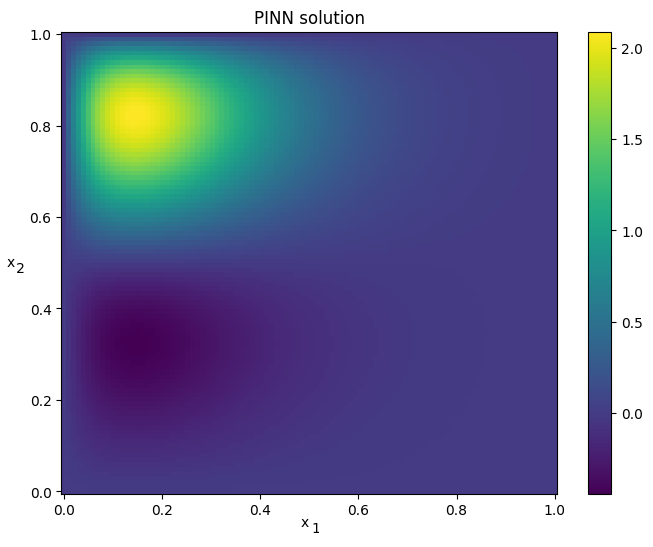}
  \caption{Solution obtained from CRVPINN for the Laplace problem with sin-exp right-hand side.}
  \label{fig:RPINN_Example2_result}
\end{figure}


The convergence of training with ADAM optimizer \cite{c26} is presented in Figure \ref{fig:RPINN_Example2_convergence}. We can see that our loss is robust and equal to the true error computed in (\ref{eq:H1}) norm. Again, for the Laplace problem  $\mu=\alpha=1$ and $\sqrt{LOSS(\theta)}=\|u_{EXACT}-u_{\theta}\|_{H^1_0(\Omega_h)}$.
The obtained solution is presented in Figure \ref{fig:RPINN_Example2_result}.

\subsection{Two-dimensional advection-diffusion problem}
\label{sec:PINN2D3}

Given $\Omega=(0,1)^2\subset\mathbb R^2$ we seek the solution $\Omega \ni (x_1,x_2) \rightarrow u(x_1,x_2)\in {\mathbb R}$ of the Eriksson-Johnson model problem \cite{c9}, a challenging model problem designed for verification of the numerical methods.

 \begin{equation}
 \left\{
 \begin{array}{rl}
 \beta \cdot \nabla u - \epsilon \Delta u = 0 & \hbox{in }\Omega\\
 u=g & \hbox{over }\partial\Omega\,,
 \end{array}\right.
 \label{eq:EJ2D}
 \end{equation}
 with $\beta=(1,0)$, $\epsilon=0.1$, 
with $g$ such that 
\begin{eqnarray}
g(0, x_2) = \sin \left(\pi x_2\right) \textrm{ for } x_2\in (0,1) \\
g(1, x_2) = 0 \textrm{ for } x_2\in (0,1) 
 \\
g(x_1, 0) = 0 \textrm{ for } x_1\in (0,1)  \\
g(x_1, 1) = 0 \textrm{ for } x_1\in (0,1) 
\end{eqnarray}

We define the shift $u_{shift}$ such that
 \begin{eqnarray}
 u(x_1,x_2) = u_0(x_1,x_2)+u_{shift}(x_1,x_2), \\ u_{shift}(x_1,x_2) = (1-x_1)\sin(\pi x_2)
 \label{eq:shift}
 \end{eqnarray}

We notice that $u_0(x_1,x_2)=u(x_1,x_2)-u_{shift}(x_1,x_2)=0$ for $(x_1,x_2) \in \partial \Omega$.
Using the shift technique, we can transform our problem to homogenous zero Dirichlet b.c. problem: we seek $\Omega \ni (x_1,x_2) \rightarrow u_0(x_1,x_2)\in {\mathbb R}$, such that

 \begin{equation}
 \left\{
 \begin{array}{rl}
 \beta \cdot \nabla u_0 - \epsilon \Delta u_0 = -\beta \cdot \nabla u_{shift} + \epsilon \Delta u_{shift}  & \hbox{in }\Omega\\
 u=0 & \hbox{over }\partial\Omega\,,
 \end{array}\right.
 \label{eq:EJ2Da}
 \end{equation}

We define the following residual function
\begin{eqnarray}
\label{eq:RES3}
\begin{aligned}
&RES_3(\theta)= \\ & \beta \cdot \nabla u_0({\bf x}) - \epsilon \Delta u_0({\bf x}) + \beta \cdot \nabla u_{shift}({\bf x}) - \epsilon \Delta u_{shift}({\bf x}) 
\end{aligned}
\end{eqnarray}
We enforce the zero Dirichlet b.c. on the NN in a strong way,  following the ideas presented in \cite{c43}.
To estimate the true error, we use the exact solution formula from \cite{c5}
\begin{eqnarray}
\label{eq:exact3}
u_{exact}(x, y) = \frac{(e^{(r_1 (x-1))} - e^{(r_2 (x-1))})}{ (e^{(-r_1)} - e^{(-r_2)})}  \sin(\pi  y), \\ r_1 = \frac{(1 + \sqrt{(1 + 4\epsilon^2\pi^2)})}{ (2\epsilon)}, r_2 = \frac{(1 - \sqrt{(1 + 4\epsilon^2\pi^2)})}{ (2\epsilon)}.
\end{eqnarray}

We define the following loss function for CRVPINN
\begin{equation}
LOSS(\theta)=RES_3(\theta)^T\times {\bf G}^{-1}\times RES_3(\theta)
\end{equation}
with $RES_3(\theta)$ defined by (\ref{eq:RES3}) and Gram matrix defined by (\ref{eq:gram}).
{As in the previous examples, the computational cost of the CRVPINN loss computations is equal to the computational cost of PINN loss computations. }

\begin{figure}[!htb]
  \centering
\includegraphics[width=0.8\textwidth]{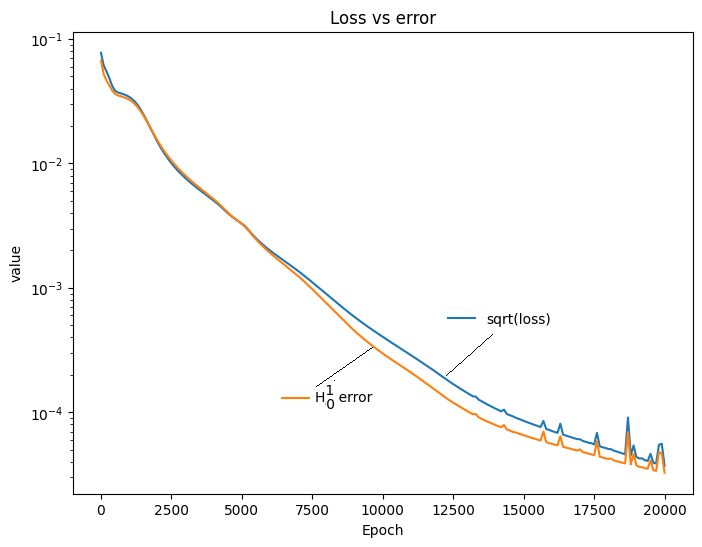}
  \caption{Convergence of CRVPINN and the true error $\epsilon H^1_0(\Omega_h)$.}
  \label{fig:RPINN_Example3_convergence}
\end{figure}

\begin{figure}[!htb]
  \centering
\includegraphics[width=0.8\textwidth]{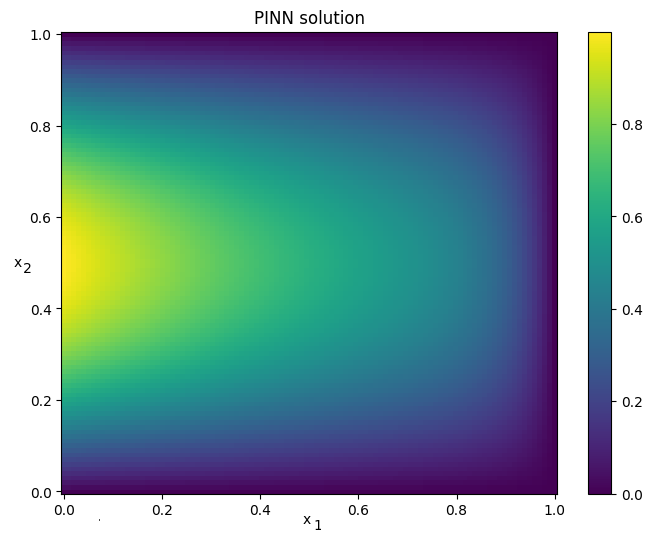}
  \caption{Solution obtained from CRVPINN for the advection-diffusion problem.}
  \label{fig:RPINN_Example3_result}
\end{figure}


The convergence of training with ADAM optimizer \cite{c26} is presented in Figure \ref{fig:RPINN_Example3_convergence}. 
For the advection-diffsion, $\mu=(\epsilon+2C)=(0.1+2\times 2)=4.1$, and $\alpha=\epsilon=0.1$. So we have $\frac{1}{4.1}\sqrt{LOSS(\theta)}\leq \|u_{EXACT}-u_{\theta}\|_{  H^1_0(\Omega_h)}\leq \frac{1}{0.1}\sqrt{LOSS(\theta)}$. Multiplying by $\epsilon=0.1$, we have  
$\frac{1}{41}\sqrt{LOSS(\theta)}\leq 0.1\times \|u_{EXACT}-u_{\theta}\|_{  H^1_0(\Omega_h)}\leq \sqrt{LOSS(\theta)}$. 
This implies the agreement of the plots if we measure the error in $\epsilon H^1_0(\Omega_h)$ norm. The robust loss function and the true error are close to each other.
The obtain solution is presented in Figure \ref{fig:RPINN_Example3_result}.

\subsection{Poisson problem with varying diffusion function}

Given $\Omega=(0,1)^2\subset\mathbb R^2$ we seek the solution of the model problem with a manufactured solution
\begin{equation}
 \nabla \cdot \left( \epsilon \left(x_2\right)  \nabla u \right) = f_4
\end{equation}
with zero Dirichlet b.c. In this problem, we select the solution
\begin{equation}
\label{eq:exact2}
u(x_1,x_2)=\sin(2\pi x_1)\sin(\pi x_2).
\end{equation}
In order to obtain this solution, we compute
\begin{eqnarray}
f_4(x_1,x_2)=
\pi \sin(\pi x_1) [ \cos(\pi x_2) d \epsilon(x_2)/(dx_2) - 2 \pi \epsilon(x_2) \sin(\pi x_2) ]
\end{eqnarray}
We assume $\epsilon(x_2)= 2 ( x_2+1 )$.
We define the following residual function
\begin{equation}
\label{eq:RES4}
RES_4(\theta)=\Delta u({\bf x})+f_4({\bf x})
\end{equation}
We enforce the zero Dirichlet b.c. on the NN in a strong way,  following the ideas presented in \cite{c43}.

The Gram matrix ${\bf G}$ is now constructed using varying $\epsilon$ values
\begin{equation}
\label{eq:varygram}
    \hat{\bf G}_{ \zeta_1 , \zeta_2 } = h ( \epsilon \nabla \delta_{ij} , \delta_{kl} ) = \begin{cases}
        2 \epsilon_{i,j}+ \epsilon_{i-1,j} + \epsilon_{i,j-1} & (k,l)=(i,j) \\
        - \epsilon_{i-1,j} & (k,l)=(i-1,j) \\
        -\epsilon_{i,j} & (k,l) = (i+1,j) \\
        -\epsilon_{i,j} & (k,l) = (i,j+1) \\
        -\epsilon_{i,j-1} & (k,l)=(i,j-1)
    \end{cases}
\end{equation}
where $\zeta_1$ is mapped into $(i,j)$ and $\zeta_2$ is mapped into $(k,l)$. 

We define the following loss function for CRVPINN
\begin{equation}
LOSS(\theta)=RES_4(\theta)^T\times \hat{{\bf G}}^{-1}\times RES_4(\theta)
\end{equation}
with $RES_4(\theta)$ defined by (\ref{eq:RES4}) and Gram matrix defined by (\ref{eq:varygram}).

\begin{figure}[!htb]
  \centering
\includegraphics[width=0.8\textwidth]{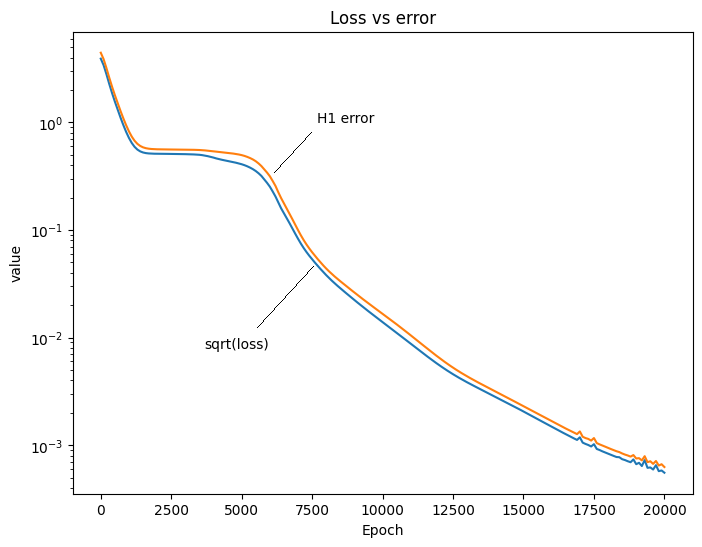}
  \caption{Convergence of CRVPINN and the true error $H^1_0(\Omega_h)$ for the Poisson problem with variable diffusion.}
  \label{fig:RPINN_Example4_convergence}
\end{figure}

\begin{figure}[!htb]
  \centering
\includegraphics[width=0.8\textwidth]{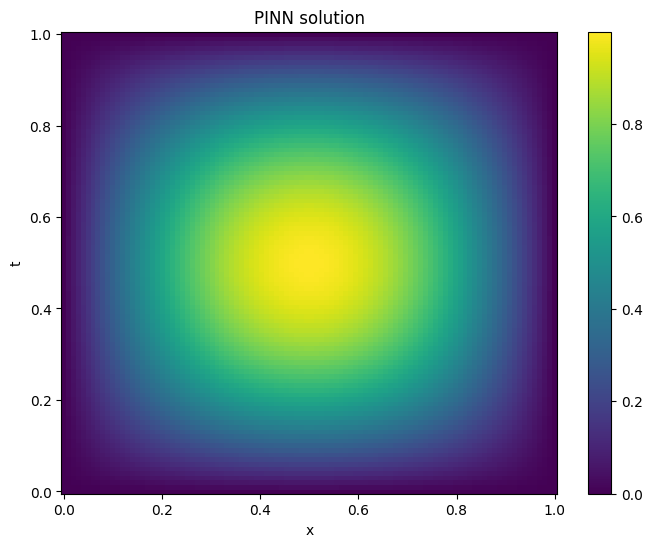}
  \caption{Solution obtained from CRVPINN for the Poisson problem with variable diffusion.}
  \label{fig:RPINN_Example4_result}
\end{figure}


\subsection{Poisson problem with a jump}

Given $\Omega=(0,1)^2\subset\mathbb R^2$ we seek the solution of the model problem with a manufactured solution
\begin{equation}
 \Delta u = f_5
\end{equation}
with zero Dirichlet b.c. In this problem, we select the solution
\begin{equation}
\label{eq:exact2}
u(x_1,x_2)=(0.45 \tanh(100(x_2-0.5)) + 0.55) \sin(\pi x_1) \sin( \pi x_2).
\end{equation}
In order to obtain this solution, we compute
\begin{eqnarray}
f_5(x_1,x_2)=
\sin(\pi x_1) (\sin(\pi x_2) \notag \\ (-10.8566 - 8.88264 \tanh(100 (-0.5 + x_2))) +  \frac{1}{\left(\cosh(100 (-0.5 + x_2))\right)^2}* \notag \\ *(282.743 \cos(\pi x_2) - 9000 \sin(\pi x_2)  \tanh(100 (-0.5 + x_2))))
\end{eqnarray}
We define the following residual function
\begin{equation}
\label{eq:RES4}
RES_5(\theta)=\Delta u({\bf x})+f_5({\bf x})
\end{equation}
We enforce the zero Dirichlet b.c. on the NN in a strong way,  following the ideas presented in \cite{c43}.

\begin{figure}[!htb]
  \centering
\includegraphics[width=0.8\textwidth]{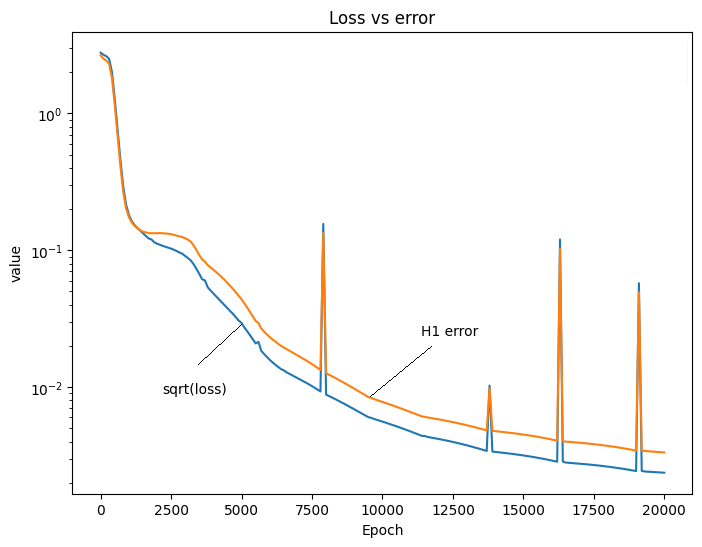}
  \caption{Convergence of CRVPINN and the true error $H^1_0(\Omega_h)$ for the Poisson problem with a jump.}
  \label{fig:RPINN_Example4_convergence}
\end{figure}

\begin{figure}[!htb]
\centering
\includegraphics[width=0.8\textwidth]{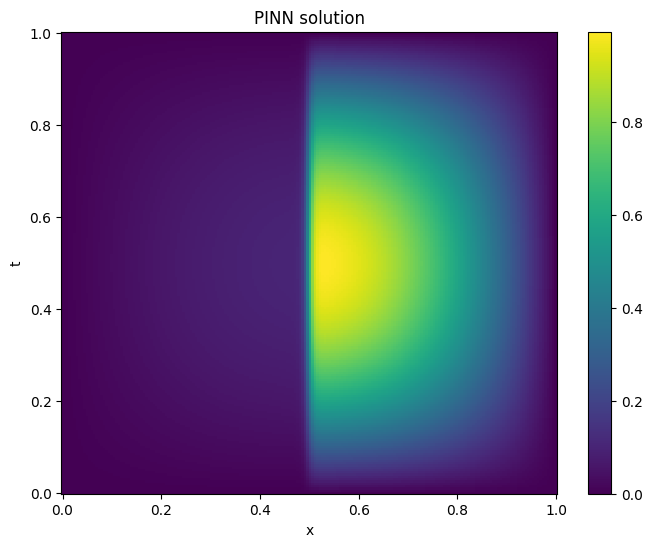}
  \caption{Solution obtained from CRVPINN for the Poisson problem with a jump.}
  \label{fig:RPINN_Example4_result}
\end{figure}


\subsection{Stokes problem with manufactured solution}

\renewcommand{\L}{\mathcal{L}}
\newcommand{\V}[1]{v\left(\bm{\lambda}^{#1}\right)}
\newcommand{\Usol}{u(\bm{\theta})}
\newcommand{\U}{\mathbf{u}}
\newcommand{\VV}{\mathbf{v}}
\newcommand{\F}{\mathbf{f}}
\newcommand{\X}{\mathbf{x}}
\newcommand{\Dpar}[2]{\frac{\partial {#1}}{\partial {#2}}}
\newcommand{\Dt}{\partial_t}

Given $\Omega=(0,1)^2\subset\mathbb R^2$ we seek the solution of the model problem with a manufactured solution: Find velocity and pressure $(u_1,u_2,p)$ such that
\begin{equation}\label{eq:strongStokes}
\begin{aligned}    
-\Delta\U+\nabla p = \F \text{ in }\Omega\\
\nabla\cdot\U= 0 \text{ in }\Omega\\
\U = {\bf g} \text{ in } \Gamma
\end{aligned}
\end{equation}
In this problem, we select the solution
\begin{equation}
\begin{aligned}    
\label{eq:exactStokes}
u_1(x_1,x_2) =& 2 e^x_1 ( - 1 + x_1)^2 x_1^2 (x_2^2 + x_2) (-1 + 2x_2), \\
u_2(x_1,x_2) =& -e^x_1  ( - 1 + x_1)  x_1  ( -2 + x_1 * (3+x_1)) (-1 + x_2)^2 x_2^2, \\
p(x_1,x_2)  =&
(-424 + 156 \cdot 2.718 + (x_2^2 - x_2) (-456 +   \\ & e^x_1(456 + x_1^2  (228 - 5(x_2^2-x_2)) 
          + 2x_1  (-228 + (x_2^2-x_2)) +\\  & 2 x_1^3 (-36 + (x_2^2-x_2)) + x_1^4 *( 12 + (x_2^2-x_2)) ) )),
\end{aligned}    
\end{equation}
and we define ${\bf g}(x_1,x_2)$ and ${\bf f}(x_1,x_2)$ accordingly, namely
\begin{equation}
\begin{aligned}    \label{rhsStokes}
f_1(x_1,x_2)&=-\frac{\partial^2 u_1(x_1,x_2)}{\partial x_1^2}-\frac{\partial^2 u_1(x_1,x_2)}{\partial x_2^2}+\frac{\partial p}{\partial x_1}, \\
f_2(x_1,x_2)&=-\frac{\partial^2 u_2(x_1,x_2)}{\partial x_1^2}-\frac{\partial^2 u_2(x_1,x_2)}{\partial x_2^2}+\frac{\partial p}{\partial x_2}, 
\end{aligned}    
\end{equation}
and 
\begin{equation}
\begin{aligned}    \label{gStokes}
g_1(x_1,x_2)&=u_1(x_1,x_2), \quad (x_1,x_2)\in \partial \Omega \\
g_2(x_1,x_2)&=u_2(x_1,x_2), \quad (x_1,x_2)\in \partial \Omega. 
\end{aligned}    
\end{equation}
We begin by transforming the Stokes equations into a first-order system:
\begin{equation}
\begin{aligned} \label{eq:stokes-first-order}
    - \nabla \cdot \Vsigma + \nabla p &= f \\
    \nabla \cdot \U &= 0 \\
   \Vsigma - \nabla \U &= 0
\end{aligned}
\end{equation}
where
\begin{equation*}
  \Vsigma = \begin{bmatrix}
      w_1 & w_2 \\
      z_1 & z_2
  \end{bmatrix},
  \quad
  \nabla \cdot \Vsigma = \begin{bmatrix}
      \deriv{w_1}{x} + \deriv{w_2}{y} \\
      \deriv{z_1}{x} + \deriv{z_2}{y}
  \end{bmatrix},
  \quad
  \nabla \U = \begin{bmatrix}
      \deriv{u_1}{x} & \deriv{u_1}{y} \\
      \deriv{u_2}{x} & \deriv{u_2}{y} \\
  \end{bmatrix}
\end{equation*}
This can be compactly written as~$Au = (f, 0, 0)$,
where
\begin{equation*}
    Au = \left( - \nabla \cdot \Vsigma + \nabla p,
    \nabla \cdot \U,
    \Vsigma - \nabla \U\right)
\end{equation*}
and~$u = (\Vsigma, \U, p)$ is a group variable.
A corresponding (continuous) variational formulation can be obtained
by testing this equality with~$v = (\Vtau, \VV, q)$:
\begin{equation}
  \Prod{Au}{v} = \Prod{-\nabla \cdot \Vsigma + \nabla p}{\VV} 
  + \Prod{\nabla\cdot\U}{q}
  + \Prod{\Vsigma - \nabla\U}{\Vtau}
  = \Prod{f}{\VV}
\end{equation}
Following \cite{c39}
we choose the following \emph{adjoint grapn norm}:
\begin{equation}
\begin{aligned}
\label{eq:stokes-norm}
  \Norm[\text{graph}]{(\boldsymbol{\tau}, \VV, q)}^2 &= \Norm{\nabla \cdot \boldsymbol{\tau} - \nabla q}^2
  + \Norm{\nabla \cdot \VV}^2 + \Norm{\boldsymbol{\tau} + \nabla \VV}^2 
  \\&+ \Norm{\boldsymbol{\tau}}^2 + \Norm{\VV}^2 + \Norm{q}^2
\end{aligned}
\end{equation}
for our test space, since then we can prove that the bilinear form~$b(u, v) := \Prod{Au}{v}$
satisfies the inf-sup condition (see Appendix in \cite{c39}).

\subsubsection{Discrete Stokes formulation}

The discrete equivalent of the above first order system can be constructed
by replacing the nabla operator with its discrete equivalents, $\nabla_{+}$
or~$\nabla_{-}$.
%
\begin{equation}
\begin{aligned} \label{eq:stokes-first-order}
    - \nabla_{+} \cdot \Vsigma + \nabla_{+} p &= f \\
    \nabla_{-} \cdot \U &= 0 \\
   \Vsigma - \nabla_{-} \U &= 0
\end{aligned}
\end{equation}

To ensure well-posedness of the above systems, we need to correctly
define the domain of its operator. Let
\begin{equation}
\begin{aligned}
  D^p_h &= \left\{p \in D_h : p|_{\Gamma_p} = 0, \Prod[h]{p}{1} = 0\right\} \\
  D^{\sigma}_h &= \left\{\Vsigma \in D_h^4 \colon \sigma_{ij}|_{\Gamma_\sigma^{j}} = 0 \right\} \\
\end{aligned}
\end{equation}
where
\begin{equation}
\begin{aligned}
    \Gamma_p &= \left\{ (0, jh) : 0 \leq j \leq N \right\}
    \cup \left\{ (ih, 0) : 0 \leq i \leq N \right\}
    \cup \{(1, 1)\} \subset \partial \Omega_h 
    \\ 
    \Gamma_\sigma^{1} &= \left\{(0, jh) : 0 \leq j \leq N \right\} 
    \\
    \Gamma_\sigma^{2} &= \left\{(ih, 0) : 0 \leq i \leq N \right\}  \\ 
\end{aligned}
\end{equation}
A corresponding (discrete) weak formulation can be obtained
by testing this equality with ~$v = (\Vtau, \VV, q)$.
We are looking for $u = (\Vsigma, \U, p) \in D^\sigma_h \times D_{0,h}^2 \times D^p_h$, such that for all~$v = (\Vtau, \VV, q) \in D^\sigma_h \times D_{0,h}^2 \times D^p_h$
\begin{equation}
  \Prod[h]{Au}{v} = \Prod[h]{-\nabla_{+} \cdot \Vsigma + \nabla_{+} p}{\VV} 
  + \Prod[h]{\nabla_{-}\cdot\U}{q}
  + \Prod[h]{\Vsigma - \nabla_{-}\U}{\Vtau}
  = \Prod[h]{f}{\VV}
\end{equation}
We consider the above formulation with test and trial spaces
consisting of the same functions:
\begin{equation}
   U = V = D^\sigma_h \times D_{0,h}^2 \times D^p_h,
\end{equation}
but with different norms:
\begin{equation}
\begin{aligned}
    \Norm[U]{u}^2 &= \Norm[h]{\Vsigma}^2 + \Norm[h]{\U}^2 + \Norm[h]{p}^2 \\
    \Norm[V]{v}^2 &= \Norm{\nabla \cdot \Vtau_{+} - \nabla_{+} q}^2
  + \Norm{\nabla_{-} \cdot \VV}^2 + \Norm{\Vtau + \nabla_{-} \VV}^2 
    \\&+ \Norm[h]{\Vtau}^2 + \Norm[h]{\VV}^2 + \Norm[h]{q}^2
\end{aligned}
\end{equation}

The corresponding discrete scalar product is
\begin{equation}
\begin{aligned}
\label{eq:stokes-scalar-product}
(u, v)_{\text{graph}} &= \Prod[h]{\nabla_{+} \cdot \Vsigma - \nabla_{+} p}{\nabla_{+} \cdot \Vtau - \nabla_{+} q}
+ \Prod[h]{\nabla_{-} \cdot \U}{\nabla_{-} \cdot \VV} \\&+
  \Prod[h]{\Vsigma + \nabla_{-} \U}{\Vtau + \nabla_{-} \VV}
  + \Prod[h]{\Vsigma}{\Vtau} + \Prod[h]{\U}{\VV} + \Prod[h]{p}{q} 
  \\&=
  \Prod[h]{\nabla_{+}\cdot \Vsigma}{\nabla_{+} \cdot \Vtau} + 2\Prod[h]{\Vsigma}{\Vtau}
  \\&+
  \Prod[h]{\nabla_{-}\cdot \U}{\nabla_{-} \cdot \VV} + \Prod[h]{\nabla_{-} \U}{\nabla_{-} \VV} + \Prod[h]{\U}{\VV}
  \\&+
  \Prod[h]{\nabla_{+} p}{\nabla_{+} q} + \Prod[h]{p}{q}
  \\&+
  \Prod[h]{\nabla_{-} \U}{\Vtau} + \Prod[h]{\Vsigma}{\nabla_{-} \VV}
  \\&+
  \Prod[h]{-\nabla_{+} p}{\nabla_{+} \cdot \Vtau} + \Prod[h]{\nabla_{+}\cdot\Vsigma}{-\nabla_{+} q}
\end{aligned}
\end{equation}
Its Gram matrix has a block structure
\begin{equation*}
    G = \begin{bmatrix}
        G_{\Vsigma} & G_{\Vsigma \U} & G_{\Vsigma p} \\
        G_{\Vsigma \U}^T & G_\U & 0 \\
        G_{\Vsigma p}^T & 0 & G_p
    \end{bmatrix}
\end{equation*}
where~$G_{\Vsigma}$, $G_\U$, $G_p$, $G_{\Vsigma \U}$, $G_{\Vsigma p}$
are matrices of the following bilinear forms, corresponding to terms of~\eqref{eq:stokes-scalar-product}:
\begin{equation}
\begin{aligned}
g_{\Vsigma}(\Vsigma, \Vtau) &=  \Prod[h]{\nabla_{+}\cdot \Vsigma}{\nabla_{+} \cdot \Vtau} + 2\Prod[h]{\Vsigma}{\Vtau} \\
g_{\U}(\U, \VV) &= \Prod[h]{\nabla_{-}\cdot \U}{\nabla_{-} \cdot \VV} + \Prod[h]{\nabla_{-} \U}{\nabla_{-} \VV} + \Prod[h]{\U}{\VV} \\
g_p(p, q) &=  \Prod[h]{\nabla_{+} p}{\nabla_{+} q} + \Prod[h]{p}{q} \\
g_{\Vsigma \U}(\U, \Vtau) &= \Prod[h]{\nabla_{-}\U}{\Vtau} \\
g_{\Vsigma p}(p, \Vtau) &= \Prod[h]{-\nabla_{+} p}{\nabla_{+}\cdot\Vtau}
\end{aligned}
\end{equation}
Since~$\Vsigma$ and~$\U$ are vectors (tensors), apart from~$G_p$, all the above matrices
can be further decomposed into blocks.
To simplify the presentation, let us introduce the following building blocks -- matrices
of fundamental bilinear forms defined on pairs of scalar functions:
\begin{equation}
\begin{gathered}
  M \sim \Prod[h]{f}{g}, \quad K_{\pm} \sim \Prod[h]{\nabla_{\pm} f}{\nabla_{\pm} g}, \quad
  S_{\pm} \sim \Prod[h]{\nabla_{x\pm} f}{\nabla_{y\pm} g} \\
  K^x_{\pm} \sim \Prod[h]{\nabla_{x\pm} f}{\nabla_{x\pm} g}, \quad K^y_{\pm} \sim \Prod[h]{\nabla_{y\pm} f}{\nabla_{y\pm} g}. \\
  A^x_{\pm} \sim \Prod[h]{\nabla_{x\pm} f}{g}, \quad A^y_{\pm} \sim \Prod[h]{\nabla_{y\pm} f}{g}.
\end{gathered}
\end{equation}
Using this notation, $G_p$ can be expressed as~$G_p = K + M$.
For~$G_{\Vsigma}$, expanding the definition with
\begin{equation*}
  \Vsigma = \begin{bmatrix}
      \sigma_{11} & \sigma_{12} \\
      \sigma_{21} & \sigma_{22}
  \end{bmatrix},
  \quad
  \Vtau = \begin{bmatrix}
      \tau_{11} & \tau_{12} \\
      \tau_{21} & \tau_{22}
  \end{bmatrix},
\end{equation*}
we get
\begin{equation*}
\begin{aligned}
    g_{\Vsigma}(\Vsigma, \Vtau) &=  \Prod[h]{\nabla_{+}\cdot \Vsigma}{\nabla_{+} \cdot \Vtau} + 2\Prod[h]{\Vsigma}{\Vtau} 
    \\&= \Prod[h]{
    \begin{bmatrix}
        \nabla_{x+} \sigma_{11} + \nabla_{y+} \sigma_{12} \\
        \nabla_{x+} \sigma_{21} + \nabla_{y+} \sigma_{22} \\
    \end{bmatrix}
    }{
        \begin{bmatrix}
        \nabla_{x+} \tau_{11} + \nabla_{y+} \tau_{12} \\
        \nabla_{x+} \tau_{21} + \nabla_{y+} \tau_{22} \\
    \end{bmatrix}
    }
    \\&
    + 2 \Prod[h]{\sigma_{11}}{\tau_{11}}
    + 2 \Prod[h]{\sigma_{12}}{\tau_{12}}
    + 2 \Prod[h]{\sigma_{21}}{\tau_{21}}
    + 2 \Prod[h]{\sigma_{22}}{\tau_{22}}
    \\&=
    \Prod[h]{\nabla_x \sigma_{11}}{\nabla_{x+} \tau_{11}}
    +
    \Prod[h]{ \nabla_{y+} \sigma_{12}}{\nabla_{x+} \tau_{11}}
    \\&+
    \Prod[h]{\nabla_{x+} \sigma_{11}}{\nabla_{y+} \tau_{12}}
    +
    \Prod[h]{\nabla_{y+} \sigma_{12}}{\nabla_{y+} \tau_{12}}
    \\&+
    \Prod[h]{\nabla_{x+} \sigma_{21}}{\nabla_{x+} \tau_{21}}
    +
    \Prod[h]{ \nabla_{y+} \sigma_{22}}{\nabla_{x+} \tau_{21}}
    \\&+
    \Prod[h]{\nabla_{x+} \sigma_{21}}{\nabla_{y+} \tau_{22}}
    +
    \Prod[h]{\nabla_{y+} \sigma_{22}}{\nabla_{y+} \tau_{22}}
    \\&
    + 2 \Prod[h]{\sigma_{11}}{\tau_{11}}
    + 2 \Prod[h]{\sigma_{12}}{\tau_{12}}
    + 2 \Prod[h]{\sigma_{21}}{\tau_{21}}
    + 2 \Prod[h]{\sigma_{22}}{\tau_{22}}
\end{aligned}
\end{equation*}
which, assuming the components are ordered as~$\sigma_{11}$, $\sigma_{12}$, $\sigma_{21}$, $\sigma_{22}$,
gives us the block structure
\begin{equation}
  G_{\Vsigma} = \begin{bmatrix}
    K^x_{+} & S^T_{+} & 0 & 0 \\
    S_{+} & K^y_{+} & 0 & 0 \\
    0 & 0 & K^x_{+} & S^T_{+} \\
    0 & 0 & S_{+} & K^y_{+} \\
  \end{bmatrix}
  + 2 \begin{bmatrix}
  M & 0 & 0 & 0 \\
  0 & M & 0 & 0\\
  0 & 0 & M & 0 \\
  0 & 0 & 0 & M
  \end{bmatrix}
\end{equation}
For~$g_{\U}$, we have
\begin{equation*}
\begin{aligned}
  g_{\U}(\U, \VV) &= \Prod[h]{\nabla_{-}\cdot \U}{\nabla_{-} \cdot \VV} + \Prod[h]{\nabla_{-} \U}{\nabla_{-} \VV} + \Prod[h]{\U}{\VV}
  \\&= \Prod[h]{\nabla_{x-} u_1 + \nabla_{y-} u_2}{\nabla_{x-} v_1 + \nabla_{y-} v_2} \\&+
  \Prod[h]{\nabla_{x-} u_1}{\nabla_{x-} v_1} + 
  \Prod[h]{\nabla_{y-} u_1}{\nabla_{y-} v_1} \\ &+
  \Prod[h]{\nabla_{x-} u_2}{\nabla_{x-} v_2} +
  \Prod[h]{\nabla_{y-} u_2}{\nabla_{y-} v_2}
  \\&+
  \Prod[h]{u_1}{v_1} +\Prod[h]{u_2}{v_2} \\&=
  2\Prod[h]{\nabla_{x-} u_1}{\nabla_{x-} v_1} + \Prod[h]{\nabla_{y-} u_1}{\nabla_{y-} v_1} + \Prod[h]{u_1}{v_1} \\&+
  \Prod[h]{\nabla_{x-} u_2}{\nabla_{x-} v_2} + 2\Prod[h]{\nabla_{y-} u_2}{\nabla_{y-} v_2} + \Prod[h]{u_2}{v_2} \\&+
  \Prod[h]{\nabla_{x-} u_1}{\nabla_{y-} v_2} \\&+ 
  \Prod[h]{\nabla_{y-} u_2}{\nabla_{x-} v_1}
\end{aligned}
\end{equation*}
which gives us the block structure
\begin{equation}
  G_{\U} = \begin{bmatrix}
    2K^x_{-} + K^y_{-} & S^T_{-} \\
    S_{-} & K^x_{-} + 2K^y_{-} \\
  \end{bmatrix}
  +
  \begin{bmatrix}
      M & 0 \\
      0 & M
  \end{bmatrix}
\end{equation}
For~$g_{\Vsigma \U}$, we get
\begin{equation*}
\begin{aligned}
    g_{\Vsigma \U}(\U, \Vtau) &= \Prod[h]{\nabla_{-}\U}{\Vtau}
    \\&=
    \Prod[h]{\nabla_{x-} u_1}{\tau_{11}} +
    \Prod[h]{\nabla_{y-} u_1}{\tau_{12}} +
    \Prod[h]{\nabla_{x-} u_2}{\tau_{21}} +
    \Prod[h]{\nabla_{y-} u_2}{\tau_{22}}
\end{aligned}
\end{equation*}
which gives us the block structure
\begin{equation}
  G_{\Vsigma \U} = \begin{bmatrix}
    A^x_{-} & 0 \\
    A^y_{-} & 0 \\
    0 & A^x_{-} \\
    0 & A^y_{-}
  \end{bmatrix}
\end{equation}
Finally, for~$g_{\Vsigma p}$, we get
\begin{equation*}
\begin{aligned}
  g_{\Vsigma p}(p, \Vtau) &= \Prod[h]{-\nabla_{+} p}{\nabla_{+}\cdot\Vtau}
  \\&= -\Prod[h]{\nabla_{x+} p}{\nabla_{x+} \tau_{11} + \nabla_{y+} \tau_{12}}
  - \Prod[h]{\nabla_{y+} p}{\nabla_{x+} \tau_{21} + \nabla_{y+} \tau_{22}}
  \\&=
  -\Prod[h]{\nabla_{x+} p}{\nabla_{x+} \tau_{11}} - \Prod[h]{\nabla_{x+} p}{\nabla_{y+} \tau_{12}}
  \\&
  \phantom{=}-\Prod[h]{\nabla_{y+} p}{\nabla_{x+} \tau_{21}} - \Prod[h]{\nabla_{y+} p}{\nabla_{y+} \tau_{22}}
\end{aligned}
\end{equation*}
\begin{equation}
  G_{\Vsigma p} = -\begin{bmatrix}
     K^x_{+} \\ S_{+} \\ S^T_{+} \\ K^y_{+}
  \end{bmatrix}
\end{equation}

\subsubsection{Stability of discrete Stokes formulation}

Stability of our discrete formulation can be investigated by establishing bounds
on the continuity and \emph{inf-sup} constants of operator~$A$,
i.e. by showing existence of~$0 <\alpha, \mu$ such that
\begin{equation}
\alpha \Norm[V]{v} 
\leq
\sup_{\substack{u \in U\\ u \neq 0}}\frac{\Prod[h]{Au}{v}}{\Norm[U]{u}}
\leq
\mu \Norm[V]{v} \quad \forall v \in V
\end{equation}
Since
\begin{equation*}
\begin{aligned}
  \Prod[h]{Au}{v} = \Prod[h]{u}{A^*v} &\leq \Norm[h]{u}\Norm[h]{A^*v}
  \\&\leq \Norm[U]{u}\sqrt{\Norm[h]{v}^2 + \Norm[h]{A^*v}^2}
  \\&= \Norm[U]{u}\Norm[V]{v},
\end{aligned}
\end{equation*}
the right part of the desired inequality holds with~$\mu = 1$.
On the other hand, we only managed to investigate the 
value of the \emph{inf-sup} constant~$\alpha$ numerically.

Let us start by recasting the problem in terms of matrices.
The operator~$A$ can be naturally extended to~$\widetilde{A}$
acting on~$\widetilde{U} = D_h^\sigma \times D_{0,h}^2 \times \widetilde{D_h^p}$,
where
\begin{equation}
\widetilde{D_h^p} = \left\{p \in D_h : p|_{\Gamma_p} = 0\right\},
\end{equation}
that is, the space without zero mean pressure constraint.
It can be proved that the kernel of~$\widetilde{A}$ consists of
zero mean pressures:
\begin{equation}
\ker \widetilde{A} = \left\{(0, 0, c) :c \in \mathbb{R}\right\}
\end{equation}
which is orthogonal to~$U \subset \widetilde{U}$.
Since clearly~$U + \ker \widetilde{A} = \widetilde{U}$,
we have~$\left(\ker \widetilde{A}\right)^\perp = U$.
Similar domain extension can be applied to~$A^*$
and the norms of~$U$ and~$V$,
giving us~$\widetilde{U}$ and~$\widetilde{V}$ with simple bases
consisting of delta functions.

Let~$G$ and~$M$ denote the Gram matrices of the scalar products
of~$\widetilde{V}$ and~$\widetilde{U}$, respectively,
and let us use~$B$ to refer to the
matrix representing the bilinear form~$u, v \mapsto \Prod[h]{\widetilde{A}u}{v}$.
We will use~$u, v$ to refer to both elements of~$\widetilde{U}$ and~$\widetilde{V}$,
and their vector representations, as long as there is no confusion.
We are interested in computing
\begin{equation}
  \alpha =
  \inf_{\substack{v \in V\\ v \neq 0}}
  \sup_{\substack{u \in U\\ u \neq 0}}
  \frac{v^T B u}{\Norm[V]{v}\Norm[U]{u}}
  =
  \inf_{\substack{v \in V\\ \Norm[V]{v} = 1}}
  \sup_{\substack{u \in U\\ \Norm[U]{u} = 1}}
  v^T B u
\end{equation}
We can compute the inner supremum for a fixed~$v$ by solving a constrained
optimization problem:
\begin{equation}
\begin{aligned}
  \textbf{max }&\ v^T B u \\
  \textbf{s.t.}&\ \|u\|_U^2 = u^T M u = 1
\end{aligned}
\end{equation}
Given~$u \in \widetilde{U}$, we can write it as~$u = u_0 + w$,
$u_0 \in U$, $w \in \ker \widetilde{A}$,
and we have~$v^T Bu = v^TBu_0$, $\|u\|_U^2 = \|u_0\|_U^2 + \|w\|^2_{\widetilde{U}}$,
which shows that the sought maximum must be attained by an element of~$U$.
Therefore, we can safely seek the maximum on the entire~$\widetilde{U}$.
The above problem gives us a Lagrangian function
\begin{equation}
\label{eq:matrix-infsup}
  \mathcal{L}(u, \lambda) = v^T Bu - \lambda\left(u^T M u - 1\right)
\end{equation}
whose stationary points satisfy
\begin{equation}
  0 = \nabla_u {\cal L} = v^T B - 2\lambda u^T M
\end{equation}
or, after transposition,
\begin{equation}
  M u = 2 \lambda B^T v
\end{equation}
Absolute value of~$\lambda$ can be chosen so that~$\Norm[U]{u} = 1$,
and the sign determines whether such~$u$ is a minimum or a maximum of
problem~$\eqref{eq:matrix-infsup}$.
Choosing~$u$ with~$\lambda > 0$ and plugging it into~\eqref{eq:matrix-infsup}
we obtain
\begin{equation}
\begin{aligned}
  \alpha &=
  \inf_{\substack{v \in V\\ v \neq 0}}
  \frac{v^T B \left(M^{-1}B^T v\right)}{
    \sqrt{v^T G v}
    \sqrt{\left(M^{-1}B v\right)^T M \left(M^{-1} B^T v\right)}}
    \\&=
    \inf_{\substack{v \in V\\ v \neq 0}}
    \sqrt{\frac{v^T B M^{-1}B^T v}{v^T G v}}
\end{aligned}
\end{equation}
Let~$R = B M^{-1}B^T$.
Finding the value of~$\alpha$ is therefore equivalent to minimizing the
generalized Rayleigh quotient
\begin{equation} \alpha^2=
\inf_{\substack{v \in V\\ v \neq 0}} \frac{v^T R v}{v^T G v}
\end{equation}
Since~$R$, $G$ are symmetric,
$G$ is positive definite
and~$R$ is positive semidefinite,
the generalized eigenvalue problem
\begin{equation}
\label{eq:gen-eigen}
  Rv = \lambda Gv
\end{equation}
posed on~$\widetilde{V}$ has a solution consisting of
non-negative eigenvalues
\begin{equation}
  0 \leq \lambda_0 \leq \lambda_1 \leq \cdots \leq \lambda_N
\end{equation}
and the associated eigenvectors~$\left\{v_i\right\}_{i=0}^N$
form a basis of~$\widetilde{V}$,
orthonormal with respect to~$G$.
As~$R = B M^{-1} B^T$, and~$B$ represents the operator~$\widetilde{A}$
with a non-trivial, one-dimensional kernel, $\lambda_0 = 0$ and~$v_0$
spans~$\ker\widetilde{A} = \ker\widetilde{A}^*$.
Since~$v_1,\ldots,v_N$ are orthogonal to~$v_0$,
we have~$\Span\left\{v_1,\ldots,v_N\right\} = U$.
We can thus write an arbitrary element of~$V$ as~$v = c_1 v_1 + \cdots + c_N v_N$,
and conclude that
\begin{equation}
\frac{v^T R v}{v^T G v} =
\frac{\lambda_1 c_1^2 + \cdots + \lambda_N c_N^2}{c_1^2 + \cdots + c_N^2}
\geq
\lambda_1
\end{equation}
In other words, the value of the \emph{inf-sup} constant is~$\alpha = \sqrt{\lambda_1}$,
and thus it can be determined numerically
by computing the second smallest eigenvalue of the problem~\eqref{eq:gen-eigen}.
Such computation reveals that $\alpha \geq 1/8$ for all the checked grid sizes, see Figure \ref{fig:spectral}. Notice that we do not really need to compute spectrum of matrix ${\bf R}$, we only compute $\lambda_1$ for the verification of the  CRVPINN method.

\begin{figure}
 \centering
\includegraphics[width=0.8\textwidth]{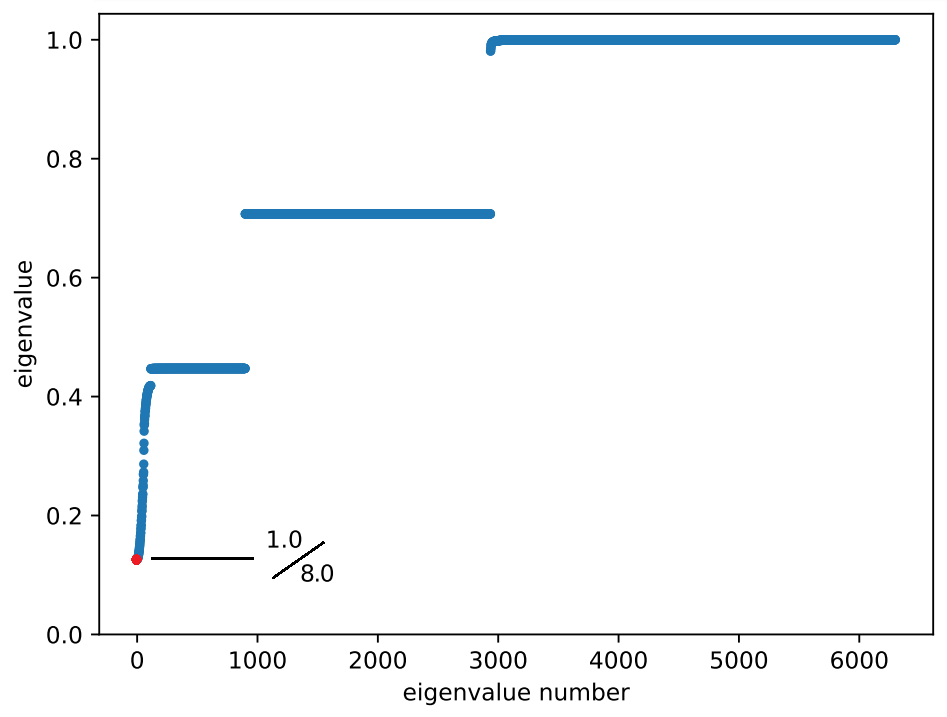}
\caption{Spectral decomposition of matrix ${\bf R}$.}
\label{fig:spectral}
\end{figure}

 \subsubsection{Robust loss for Stokes problem}

System~\eqref{eq:strongStokes} can be rewritten as
\begin{equation}
\begin{aligned}    \label{decomposedStokes}
w_1(x_1,x_2)& =\frac{\partial u_1(x_1,x_2)}{\partial x_1} \\
w_2(x_1,x_2)& =\frac{\partial u_1(x_1,x_2)}{\partial x_2} \\
z_1(x_1,x_2)& =\frac{\partial u_2(x_1,x_2)}{\partial x_1} \\
z_2(x_1,x_2)& =\frac{\partial u_2(x_1,x_2)}{\partial x_2} \\
-\frac{\partial w_1(x_1,x_2)}{\partial x_1}-\frac{\partial w_2(x_1,x_2)}{\partial x_2}+\frac{\partial p(x_1,x_2)}{\partial x_1} &= f_1(x_1,x_2), \\
-\frac{\partial z_1(x_1,x_2)}{\partial x_1}-\frac{\partial z_2(x_1,x_2)}{\partial x_2}+\frac{\partial p(x_1,x_2)}{\partial x_2} &= f_2(x_1,x_2), \\
\frac{\partial u_1(x_1,x_2)}{\partial x_1}+\frac{\partial u_2(x_1,x_2)}{\partial x_2}&=0.
\end{aligned}
\end{equation}
We define the following residual functions
\begin{equation}
\begin{aligned}
\label{eq:RESStokes}
RES_{6a}(u_\theta)=\frac{\partial u_1}{\partial x_1}-w_1, \\
RES_{6b}(u_\theta)=\frac{\partial u_1}{\partial x_2}-w_2 , \\
RES_{6c}(u_\theta)=\frac{\partial u_2}{\partial x_1}-z_1, \\
RES_{6d}(u_\theta)=\frac{\partial u_2}{\partial x_2}-z_2, \\
RES_{6e}(u_\theta)=-\frac{\partial w_1}{\partial x_1}-\frac{\partial w_2}{\partial x_2}+\frac{\partial p}{\partial x_1}-f_1, \\
RES_{6f}(u_\theta)=-\frac{\partial z_1}{\partial x_1}-\frac{\partial z_2}{\partial x_2}+\frac{\partial p}{\partial x_2}-f_2
, \\
RES_{6g}(u_\theta)=\frac{\partial u_1(x_1,x_2)}{\partial x_1}+\frac{\partial u_2(x_1,x_2)}{\partial x_2}. \\
\end{aligned} 
\end{equation}
We define the following robust loss
\begin{equation}
LOSS(u_\theta)=RES(u_\theta)^T
    \begin{bmatrix}
        G_{\Vsigma} & G_{\Vsigma \U} & G_{\Vsigma p} \\
        G_{\Vsigma \U}^T & G_\U & 0 \\
        G_{\Vsigma p}^T & 0 & G_p
    \end{bmatrix}^{-1} RES(u_\theta),
\end{equation}
where
\begin{equation}
RES(u_\theta)=
\begin{bmatrix}
RES_{6a}(u_\theta) \\ RES_{6b}(u_\theta) \\ 
RES_{6c}(u_\theta) \\ RES_{6d}(u_\theta) \\
RES_{6e}(u_\theta) \\ RES_{6f}(u_\theta) \\ 
RES_{6g}(u_\theta)
\end{bmatrix}.
\end{equation}

The Dirichlet boundary condition is obtained by multiplication of the output from the neural network by a smooth function equal to zero on the boundaries, see Figure \ref{fig:zero}, followed by adding 
a maximum of the four functions presented in Figure \ref{fig:StokesFunctions} multiplied by the ${\bf g}$ function (definition of the Dirichlet b.c.).

We have estimated numerically the continuity constant $\frac{1}{\mu}=1$ and the inf-sup constnat $\frac{1}{\alpha}=8$.
We compare in Figure \ref{fig:StokesConv} the robust loss of CRVPINN method with the true error and we obtain 
\begin{eqnarray}
    1\sqrt{LOSS(u_\theta)} < \|u_\theta - u_{exact} \|  < 8  \sqrt{LOSS(u_\theta)}
\end{eqnarray}
as expected.

The comparison between the exact solution and the CRVPINN solution, namely the velocity and pressure components and the error maps are presented in Figure \ref{fig:RPINN_Stokes}.

\begin{figure}[!htb]
  \centering
\includegraphics[width=0.48\textwidth]{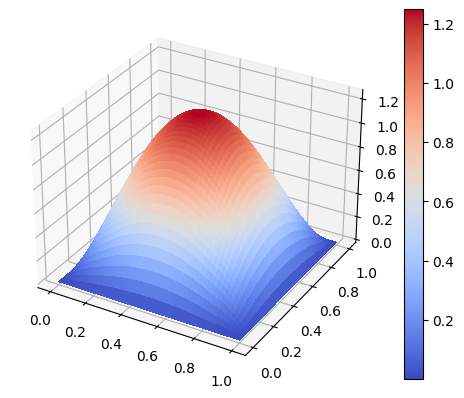}
\caption{The function used to enforce the zero on the boundary of the domain.}  \label{fig:zero}
\end{figure}

\begin{figure}[!htb]
  \centering
\includegraphics[width=0.48\textwidth]{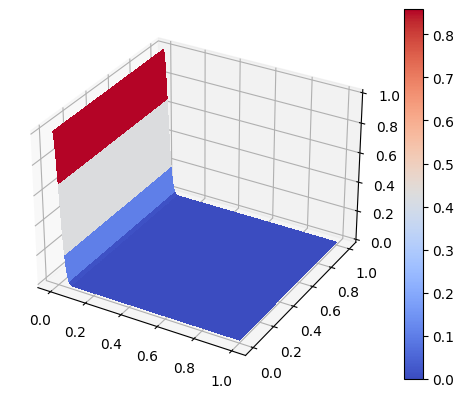}\includegraphics[width=0.48\textwidth]{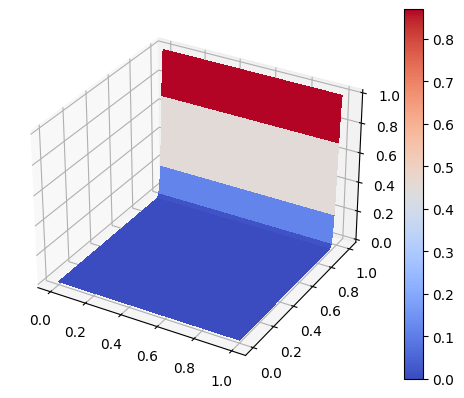}
\includegraphics[width=0.48\textwidth]{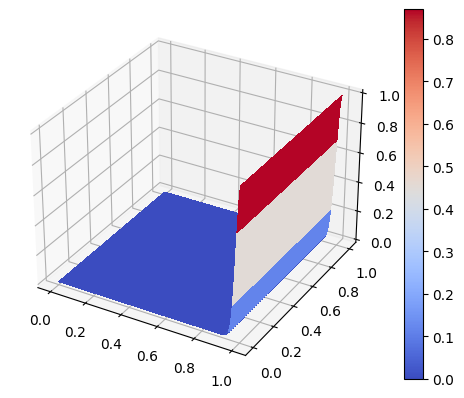}\includegraphics[width=0.48\textwidth]{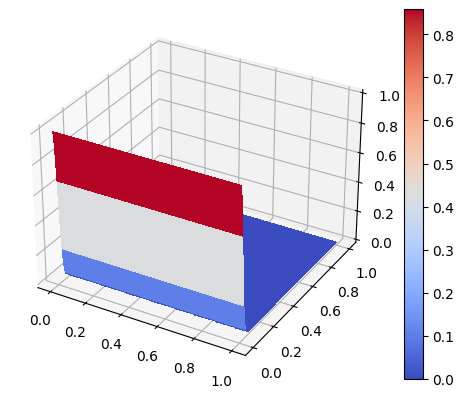}
  \caption{The functions used to enforce the Dirichlet boundary condition for a Stokes problem.}
  \label{fig:StokesFunctions}
\end{figure}

\begin{figure}[!htb]
  \centering
\includegraphics[width=
\textwidth]{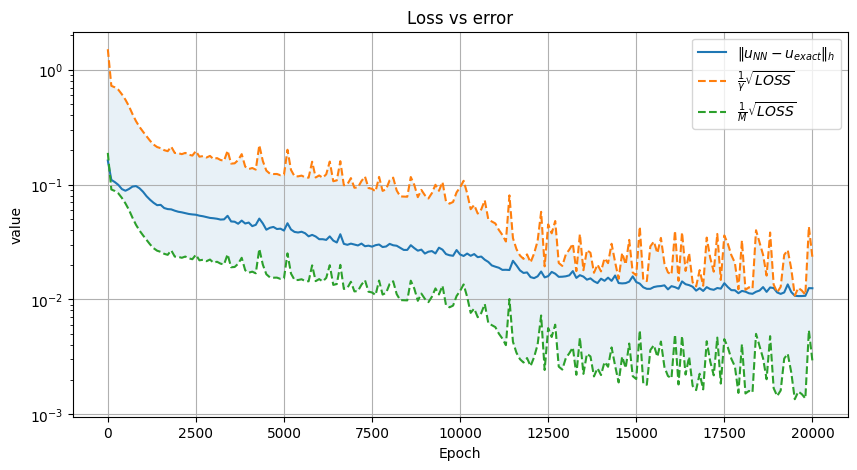}
  \caption{Robust loss and the true error for the Stokes problem with manufactured solution coincides with our estimates for the continuity constant $\frac{1}{\mu}=1$ and inf-sup constant $\frac{1}{\alpha}=8$ \\ 
  $1\sqrt{LOSS(u_\theta)} < \|u_\theta - u_{exact} \|  < 8  \sqrt{LOSS(u_\theta)}$}
  \label{fig:StokesConv}
\end{figure}

\begin{figure}[!htb]
\centering
\includegraphics[width=1.0\textwidth]{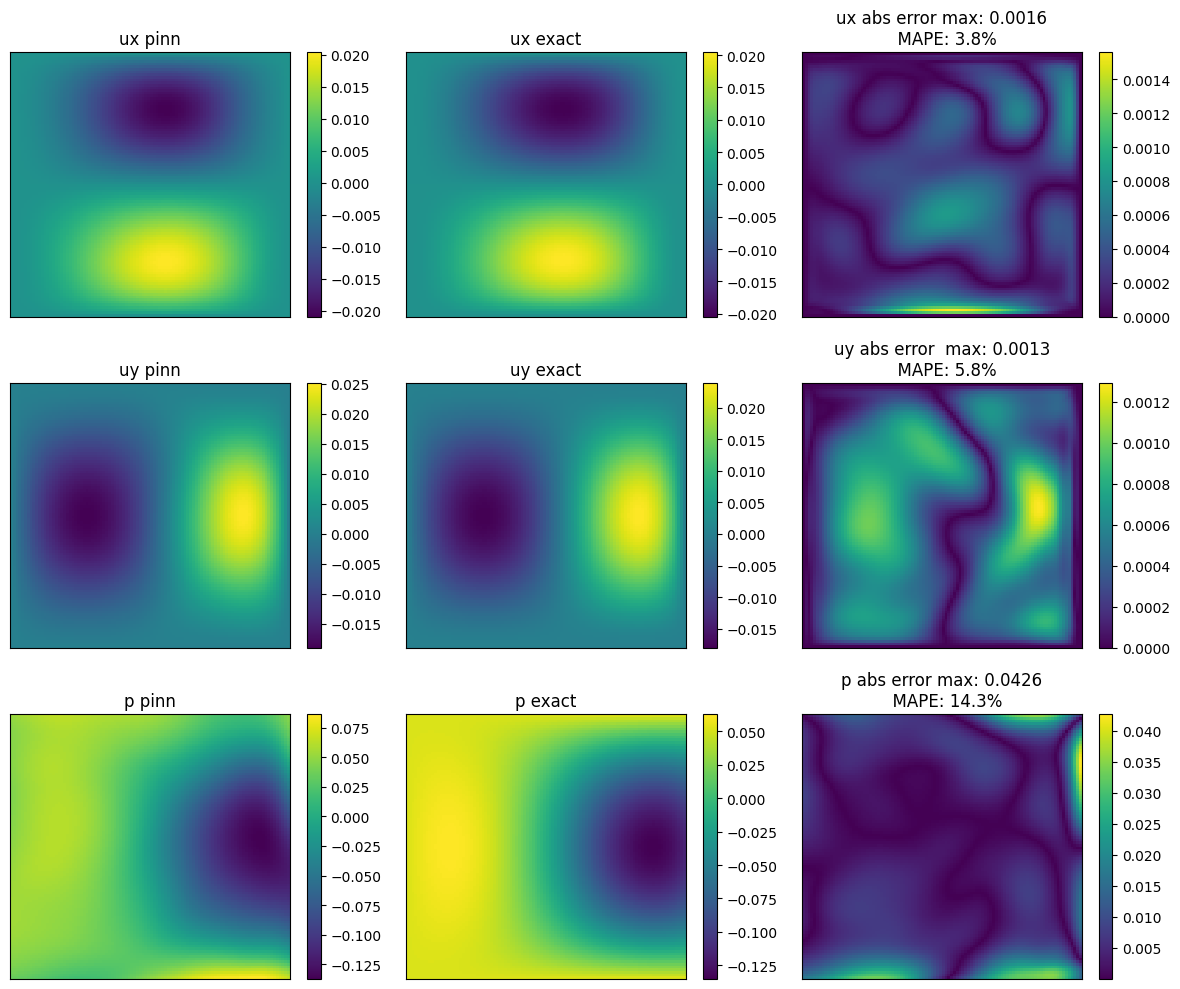}
  \caption{Comparison between the solution generated by CRVPINN and the exact solution of the Stokes problem with manufactured solution.}
  \label{fig:RPINN_Stokes}
\end{figure}



\subsection{Summary}

By saying that the loss of CRVPINN is robust, we mean that the CRVPINN loss value is the robust estimation for the true
error, which means that up to some multiplicative constant, the loss defines a lower and upper bound for the true error. This is expressed in our Remark 5. For some classes of PDEs (for example, the Laplace problem), the multiplicative constant is equal to 1. In this case, the CRVPINN loss is equal to the true error. 
Some differences present e.g. in Figure \ref{fig:RPINN_Example2_convergence} for the Laplace problem with sin-exp right-hand side may be a consequence of the fact that in our theory we compute the derivatives using the point values over the discrete domain, while in the code we differentiate the neural network using automatic differentiation provided by PyTorch.
{For the other problems, like advection-diffusion or the Stokes problem, the constants are not necesserly equal to 1, and this implies some differences in the estimates. Neverless, the robust loss is a good estimator of the true error and it can be employed for monitoring the convergence.}

Summing up, the CRVPINN
robust loss is equal to or close to the true error, the
difference between the exact and the CRVPINN solution. Thus,
while training CRVPINN, we know the quality of the solution just by looking at the value of the loss function (even if we do not know the exact solution).

\section{Computational costs}

The loss in the PINN
method involves the summation of the residual for all
the considered points. The loss in CRVPINN involves the
multiplication of the residual vector by the inverse of
the Gram matrix. The Gram matrix is generated and inverted only once at the beginning of training. Thus, the
computational overhead of CRVPINN with respect to PINN
is small. Below, we report the times for 20,000 iterations, 2 layers, and a neural network with 100 neurons
each, and a training rate of 0.001, using 100x100 points
for the training. Problem 1 Google Colab execution on V100 card for the implementation of PINN takes 200 seconds, while CRVPINN
takes 296 seconds. Problem 2 Google Colab Google Colab execution on V100 card for the  implementation of PINN takes 208 seconds, while CRVPINN
takes 301 seconds. Finally, Problem 3 Google Colab execution on V100 card for the implementation of PINN takes 228 seconds, while CRVPINN
takes 325 seconds.

\section{Conclusions}\label{Sec:Conc}
In this article, we proposed a collocation method for Robust Variational Physics Informed Neural Networks. 
We numerically show the robustness of our loss function, while having the computational cost of the method similar to PINN. For all the numerical examples, it can be used as the true error estimator.
Our robust loss function involves the inverse of the Gram matrix computed for the special inner product. The norm related to this inner product allows us to show that the bilinear form of the discrete weak formulation of our PDE is bounded and inf-sup stable. This project is the transfer of knowledge from the theory of finite difference methods into the RVPINN.

Summing up, the robust loss function for CRVPINN is given by the norm of the vector of residuals
computed at various points, induced by the inverse of the Gram matrix. Namely, 
\begin{itemize}
\item We select the PDE, e.g., the advection-diffusion, and we derive its discrete weak formulation, e.g. (\ref{eq:discweak})
\item 
We seek the inner product for which 
the form $b(u_{i,j},v_{i,j})$ of the discrete weak formulation (\ref{eq:discweak}) is a bounded inf-sup stable 
bilinear form in the induced norm. For the advection-diffusion, we select the inner product (\ref{eq:H1}). 
\item We select the test functions. They correspond with the points selected for training, and they are given by (\ref{eq:test}), namely $ \{ \delta_{i,j}(x) \}_{i,j}$.
\item We compute the Gram matrix ${\bf G}$, the inner products of the test functions. In our case, with test functions given by (\ref{eq:test}), and the inner product given by (\ref{eq:H1}), the Gram matrix is prescribed by (\ref{eq:gram}).
\item We compute LU factorization (once at the beginnging) of the sparse Gram matrix to obtained ${\bf G}={\bf L}{\bf U}$.
\item To speedup the training process, we perform forward and backward substitutions in each iteration.
The robust loss function, defined as $LOSS(\theta)=RES(\theta)^T{\bf G}^{-1}RES(\theta)$, is obtained from
backward substitution of ${\bf U} z = RES(\theta)$, followed by foward substitution of ${\bf L} q = z$, following by two-vectors multiplication
 $LOSS(\theta) = RES_1^T(\theta) q$
\end{itemize}

The future work may involve extension of the method to 
other classical problems solved by finite element method \cite{c10}, including
fluid flow problems \cite{c16},  structural analysis \cite{c7}, phase-field problems \cite{c14}, or space-time problems \cite{c11}.

In general, our CRVPINN method can be applied to a large
class of PDEs. To develop the CRVPINN formulation of
a given PDE, a proper inner product for the construction of the Gram matrix has to be designed. The CRVPINN
method uses the inner product definitions to obtain the robust loss function in a similar way as Residual Minimization \cite{c4} and Discontinuous Petrov-Galerkin \cite{c8} methods enrich the classical finite element method formulations. Our future work will
involve development of CRVPINN formulations for linear elasticity, time-harmonic Maxwell equations, as
well as for the transient and non-linear problems, including heat transfer, non-linear flow in heterogeneous media, Navier-Stokes, plasticity, transient Maxwell equations. For each of these problems, the inner product that
results in inf-sup stability of the discrete weak formulation must be developed, which will be a subject of our
future work.

\section{Acknowledgments}
This work was supported by the program ,,Excellence initiative - research university'' for the AGH University of Krakow.
This Project has received funding from the European Union’s Horizon Europe research and innovation programme under the Marie Sklodowska-Curie grant agreement No 101119556.

\appendix


\begin{thebibliography}{8}

\bibitem{c1}
 Mark Alber, Adrian Buganza Tepole, William R. Cannon, Suvranu De, Sal
vador Dura-Bernal, Krishna Garikipati, George Karniadakis, William W. Lytton, Paris Perdikaris, Linda Petzold, and Ellen Kuhl. Integrating machine learning and multiscale modeling-perspectives, challenges, and opportunities in the  biologica biomedical, and behavioral sciences. NPJ Digital Medicine, 2, 2019.
\bibitem{c2}
 Santiago Badia, Wei Li, and Alberto F. Martin. Finite element interpolated neural networks for solving forward and inverse problems. Computer Methods  in Applied Mechanics and Engineering, 418(A), 2024.
\bibitem{c3}
 Shengze Cai, Zhiping Mao, Zhicheng Wang, Minglang Yin, and George Em Karniadakis. Physics-informed neural networks (PINNs) for fluid mechanics:  A review. Acta Mechanica Sinica, 37(12):1727–1738, 2021.
\bibitem{c4}
 Jesse Chan and John A. Evans. A minimum-residual finite element method
 for the convection-diffusion equation. ICES Report 13-12, 2013.
\bibitem{c5}
 Jesse Chan, Norbert Heuer, Tan Bui-Thanh, and Leszek Demkowicz. A robust
 DPG method for convection-dominated diffusion problems II: Adjoint bound
ary conditions and mesh-dependent test norms. Computers \& Mathematics
 with Applications, 67(4):771–795, 2014.
\bibitem{c6}
Yuyao Chen, Lu Lu, George Em Karniadakis, and Luca Dal Negro. Physics
informed neural networks for inverse problems in nano-optics and metamaterials. Optics express, 28(8):11618–11633, 2020.
 34
\bibitem{c7}
J.A. Cottrell, T.J.R. Hughes, and A. Reali. Studies of refinement and continuity in isogeometric structural analysis. Computer Methods in Applied Mechanics and Engineering, 196(41):4160–4183, 2007.
\bibitem{c8}
L. Demkowicz and J. Gopalakrishnan. An overview of the discontinuous
 Petrov Galerkin method. In X. Feng, O. Karakashian, and Y. Xing, editors, Recent Developments in Discontinuous Galerkin Finite Element Methods for Partial Differential Equations: 2012 John H Barrett Memorial Lectures, volume
 157 of The IMA Volumes in Mathematics and its Applications, pages 149–180.
 Springer, Cham, 2014.
\bibitem{c9}
Kenneth Eriksson and Claes Johnson. Adaptive finite element methods for parabolic problems I: A linear model problem. SIAM J. Num. Anal., 28(1):43
77, 1991.
\bibitem{c10}
Alexandre Ern and Jean-Luc Guermond. Finite Elements I. Approximation
 and Interpolation. Springer, 2020.
\bibitem{c11}
Thomas Führer and Michael Karkulik. Space–time least-squares finite elements for parabolic equations. Computers \& Mathematics with Applications,
 92:27–36, 2021.
\bibitem{c12}
Nicholas Geneva and Nicholas Zabaras. Modeling the dynamics of PDE
 systems with physics-constrained deep auto-regressive networks. Journal of
 Computational Physics, 403, 2020.
\bibitem{c13}
Mehdi Gheisari, Guojun Wang, and Md Zakirul Alam Bhuiyan. A survey on
 deep learning in big data. In 2017 IEEE international conference on computational science and engineering (CSE) and IEEE international conference on
 embedded and ubiquitous computing (EUC), volume 2, pages 173–180. IEEE,
 2017.
\bibitem{c14}
Hector Gomez, Alessandro Reali, and Giancarlo Sangalli. Accurate, efficient, and (iso)geometrically flexible collocation methods for phase-field models. Journal of Computational Physics, 262:153–171, 2014.
\bibitem{c15}
Somdatta Goswami, Cosmin Anitescu, Souvik Chakraborty, and Timon
 Rabczuk. Transfer learning enhanced physics informed neural network for
 phase-field modeling of fracture. Theoretical and applied fracture machanics,
 106, 2020.
\bibitem{c16}
J.L. Guermond and P.D. Minev. A new class of massively parallel direction
 splitting for the incompressible Navier–Stokes equations. Computer Methods
 in Applied Mechanics and Engineering, 200(23):2083–2093, 2011.
\bibitem{c17}
Geoffrey Hinton, Li Deng, Dong Yu, George E Dahl, Abdel-rahman Mohamed, Navdeep Jaitly, Andrew Senior, Vincent Vanhoucke, Patrick Nguyen,  Tara N Sainath, et al. Deep neural networks for acoustic modeling in speech recognition: The shared views of four research groups. IEEE Signal processing
 magazine, 29(6):82–97, 2012.
\bibitem{c18}
Shenglin Huang, Zequn He, Bryan Chem, and Celia Reina. Variational Onsager Neural Networks (VONNs): A thermodynamics-based variational learning strategy for non-equilibrium PDEs. Journal of the mechanics and physics
 of solids, 163, 2022.
\bibitem{c19}
Xiang Huang, Hongsheng Liu, Beiji Shi, Zidong Wang, Kang Yang, Yang
 Li, Min Wang, Haotian Chu, Jing Zhou, Fan Yu, Bei Hua, Bin Dong, and Lei
 Chen. A universal PINNs method for solving Partial Differential Equations
 with a point source. Proceedings of the Fourteen International Joint Conference on Artificial Intelligence (IJCAI-22), pages 3839–3846, 2022.
\bibitem{c20}
Ameya D. Jagtap, Kenji Kawaguchi, and George Em Karniadakis. Adaptive activation functions accelerate convergence in deep and physics-informed
 neural networks. Journal of Computational Physics, 404, 2020.
\bibitem{c21}
Henry Jin, Marios Mattheakis, and Pavlos Protopapas. Physics-informed
 neural networks for quantum eigenvalue problems. In 2022 International Joint
 Conference on Neural Networks (IJCNN), pages 1–8, 2022.
\bibitem{c22}
Ehsan Kharazmi, Zhongqiang Zhang, and George Em Karniadakis. Variational physics-informed neural networks for solving partial differential equations. arXiv preprint arXiv:1912.00873, 2019.
\bibitem{c23}
Ehsan Kharazmi, Zhongqiang Zhang, and George E.M. Karniadakis. $hp$
VPINNs: Variational physics-informed neural networks with domain decomposition. Computer Methods in Applied Mechanics and Engineering,
 374:113547, 2021.
\bibitem{c24}
Jungeun Kim, Kookjin Lee, Dongeun Lee, Sheo Yon Jhin, Noseong Park
DPM: A novel training method for physics-informed neural networks in extrapolation. The Thirty-Fifth AAAI Conference on Artificial Intelligence, 35.

\bibitem{c25}
Younghyeon Kim, Hyungyeol Kwak, and Jaewook Nam. Physics-informed
 neural networks for learning fluid flows with symmetry. Korean Journal of
 Chemical Engineering, 40(9):2119–2127, SEP 2023.
\bibitem{c26}
Diederik P Kingma and Jimmy Ba. Adam: A method for stochastic optimization. arXiv preprint arXiv:1412.6980, 2014.
\bibitem{c27}
Georgios Kissas, Yibo Yang, Eileen Hwuang, Walter R. Witschey, John A.
 Detre, and Paris Perdikaris. Machine learning in cardiovascular flows mod
eling: Predicting arterial blood pressure from non-invasive 4D flow MRI data
 using physics-informed neural networks. Computer Methods in Applied Mechanics and Engineering, 358, 2020.
\bibitem{c28}
Alex Krizhevsky, Ilya Sutskever, and Geoffrey E Hinton. Imagenet classification with deep convolutional neural networks. Communications of the ACM,
 60(6):84–90, 2017.
\bibitem{c29}
Julia Ling, Andrew Kurzawski, and Jeremy Templeton. Reynolds averaged
 turbulence modelling using deep neural networks with embedded invariance.
 Journal of Fuild Mechanics, 807:155–166, 2016.
\bibitem{c30}
Chuang Liu and Heng An Wu. A variational formulation of physics
informed neural network for the applications of homogeneous and heteroge
neous material properties identification. International Journal of Applied Mechanics, 15(08), 2023.
\bibitem{c31}
Chuang Liu and HengAn Wu. cv-PINN: Efficient learning of variational
 physics-informed neural network with domain decomposition. Extreme Mechanics Letter, 63, 2023.
\bibitem{c32}
Lu Lu, Raphael Pestourie, Wenjie Yao, Zhicheng Wang, Francesc Verdugo, and Steven G. Johnson. Physics-informed neural networks with
 hard constraints for inverse design. SIAM Journal on Scientific Computing,
 43(6):B1105–B1132, 2021.
\bibitem{c33}
Paweł Maczuga and Maciej Paszy\'nski. Influence of activation functions
 on the convergence of physics-informed neural networks for 1d wave equation. In Jiˇrí Mikyška, Clélia de Mulatier, Maciej Paszy\'nski, Valeria V.
 Krzhizhanovskaya, Jack J. Dongarra, and Peter M.A. Sloot, editors, Compu
tational Science– ICCS 2023, pages 74–88, Cham, 2023. Springer Nature
 Switzerland.
\bibitem{c34}
Zhiping Mao, Ameya D Jagtap, and George Em Karniadakis. Physics
informed neural networks for high-speed flows. Computer Methods in Applied
 Mechanics and Engineering, 360:112789, 2020.
\bibitem{c35}
Siddhartha Mishra and Roberto Molinaro. Estimates on the generalization
 error of physics-informed neural networks for approximating a class of inverse
 problems for PDEs. IMA Journal of Numerical Analysis, 42(2):981–1022,
 2022.
\bibitem{c36}
Rahul Nellikkath and Spyros Chatzivasileiadis. Physics-informed neural networks for minimising worst-case violations in DC optimal power flow. In 2021
 IEEE International Conference on Communications, Control, and Computing
 Technologies for Smart Grids (SmartGridComm), pages 419–424, 2021.
\bibitem{c37}
Maziar Raissi, Paris Perdikaris, and George E Karniadakis. Physics
informed neural networks: A deep learning framework for solving forward and
 inverse problems involving nonlinear partial differential equations. Journal of
 Computational physics, 378:686–707, 2019.
\bibitem{c38}
Majid Rasht-Behesht, Christian Huber, Khemraj Shukla, and George Em
 Karniadakis. Physics-informed neural networks (PINNs) for wave propagation
 and full waveform inversions. Journal of Geophysical Research: Solid Earth,
 127(5):e2021JB023120, 2022.
\bibitem{c39}
Nathan V. Roberts, Tan Bui-Thanh, and Leszek Demkowicz. The DPG
 method for the stokes problem. Computers \& Mathematics with Applications,
 67(4):966–995, 2014. High-order Finite Element Approximation for Partial
 Differential Equations.
\bibitem{c40}
Sergio Rojas, Paweł Maczuga, Judit Muñoz-Matute, David Pardo, and Maciej Paszy\'nski. Robust variational physics-informed neural networks. Computer Methods in Applied Mechanics and Engineering, 425:116904, 2024.
\bibitem{c41}
Dongho Shinand, John C.Strikwerda. Inf-sup conditions for finite-difference
 approximations of the Stokes equations. The Journal of the Australian Mathematical Society. Series B. Applied Mathematics, 39(1):121–134, 1997.
 \bibitem{c42}
Luning Sun, Han Gao, Shaowu Pan, and Jian-Xun Wang. Surrogate modeling for fluid flows based on physics-constrained deep learning without simulation data. Computer Methods in Applied Mechanics and Engineering, 361,
 2020.
 
\bibitem{c43}
Luning Sun, Han Gao, Shaowu Pan, and Jian-Xun Wang. Surrogate modeling for fluid flows based on physics-constrained deep learning without simulation data. Computer Methods in Applied Mechanics and Engineering,
 361:112732, 2020.
\bibitem{c44}
 Fangzheng Sun, Yang Liu, and Hao Sun. Physics-informed spline learning
 for nonlinear dynamics discovery. Proceedings of the Thirtieth International
 Joint Conference on Artificial Intelligence (IJCAI-21), pages 2054–2061, 2021.
\bibitem{c45}
Nils Wandel, Michael Weinmann, Michael Neidlin, and Reinhard Klein.
 Spline-PINN: Approaching PDEs without data using fast, physics-informed
 Hermite-spline CNNs. Proceedings of the AAAI Conference on Artificial Intelligence, 36(8):8529–8538, 2022.
\bibitem{c46}
Yibo Yang and Paris Perdikaris. Adversarial uncertainty quantification
 in physics-informed neural networks. Journal of Computational Physics,
 394:136–152, 2019.
\end{thebibliography}

\end{document}